%% file: IEEE_TCC_1st.tex
\documentclass[A4paper,twoside,journall]{IEEEtran}

\include{Definition}

\begin{document}

\title{\huge Multi-Timescale Latent-Action DRL for Joint Optimization in Edge-Cloud Networks}
\author{
\IEEEauthorblockN{Vo Phi Son, Van-Dinh Nguyen, Ngoc Hung Nguyen, Trinh Van Chien and  Symeon Chatzinotas\\}
\thanks{V. P. Son is with Smart Green Transformation Center and College of Engineering and Computer Science, VinUniversity, Vietnam (e-mail: son.vp@vinuni.edu.vn). V.-D. Nguyen (Corresponding author) is with Trinity College Dublin,  Ireland  (e-mail: dinh.nguyen@tcd.ie). N. H. Nguyen is with Phenikaa University, Vietnam (e-mail: hung.nguyenngoc@phenikaa-uni.edu.vn). T. V. Chien is with Hanoi University of Science and Technology,  Vietnam. (e-mail: chientv@soict.hust.edu.vn). S. Chatzinotas is with SnT, University of Luxembourg, Luxembourg (e-mail: symeon.chatzinotas@uni.lu).}
\vspace{-18pt}}

\markboth{Submitted to IEEE Transactions on Cloud Computing}%
{Vo \MakeLowercase{\textit{et al.}}: Multi-Timescale Latent-Action DRL for Joint Optimization in Edge-Cloud Networks}
\maketitle 

\begin{abstract} Load imbalance across edge and cloud layers degrades latency performance in hierarchical edge–cloud computing (HECC) systems under dynamic task arrivals and heterogeneous resources, leading to severe queuing delays and inefficient resource utilization. To address this challenge, we study a joint service placement, computational delegation, and power control (JSCP) problem to minimize the average end-to-end (e2e) latency. The resulting JSCP problem is a mixed-integer nonconvex and NP-hard optimization problem due to the strong coupling between discrete and continuous variables. To enable tractable optimization and stable system adaptation, we exploit the inherent difference in decision dynamics and decompose the problem into long-term system configuration and short-term resource allocation subproblems. Based on this formulation, we propose a two-timescale multi-layer deep reinforcement learning framework with a latent action space (2T-MDRL-LA) to jointly optimize service placement, user association, computational delegation, task offloading, and user transmit power. A latent action representation based on a variational autoencoder is introduced to efficiently compress the high-dimensional combinatorial action space. Simulation results demonstrate that the proposed framework effectively adapts to dynamic network conditions and achieves near-optimal performance compared to branch-and-bound solutions. It achieves up to a $20.8\%$ reduction in average e2e latency and a $13\%$ improvement in resource utilization over the scheme without the computational delegation, while converging approximately $50\%$ faster than conventional proximal policy optimization. 

\end{abstract}


\section{Introduction}\label{sec_Intro}
Recent advances in hierarchical edge–cloud computing (HECC) for Internet of Things (IoT) have enabled a promising paradigm for latency-sensitive applications by leveraging cooperation between edge servers (ESs) and the cloud server (CS). In such systems, micro-tasks from nearby users (UEs) are processed at ESs, while computationally intensive or delay-tolerant tasks are handled at the CS, thereby reducing e2e latency~\cite{Sun_2024, Su_2023}. For example, real-time sensor data are processed at ESs to meet strict latency requirements, whereas applications such as disaster forecasting or image recognition are executed at the CS due to their high computational demands~\cite{Sun_2024, Sun_2024_iot}.

Despite these advantages, several challenges remain. Due to limited resources, ESs cannot host all services, requiring joint optimization of user association and service placement under dynamic conditions. Moreover, random task arrivals in time and space~\cite{Yang_2023,Lima_2022} lead to uneven workloads across ESs, causing queueing delays, resource underutilization, and degraded latency performance~\cite{Yang_2023}. In addition, frequent data transmission among UEs, ESs, and the CS introduces further delay bottlenecks~\cite{huynh_service_placement}. These challenges call for a unified framework that jointly optimizes user association, service placement, computational delegation, and task offloading.

Existing studies have addressed parts of this problem. Computational delegation and task offloading have been explored for UE–ES and ES–CS cooperation~\cite{huynh_service_placement,Wang_2020_Fed,Ma_2024}, but often neglect service placement and assume single-service access. Other works jointly optimize service placement and task offloading~\cite{Zhou_2023,Chu_2024,Shang_2024,Yao_2022}, but typically rely on fixed service configurations and overlook computational delegation across ESs and between ES and CS layers. As a result, these approaches may not effectively handle dynamic and large-scale IoT environments.

Furthermore, jointly optimizing user association, service placement, computational delegation, and task offloading leads to an NP-hard mixed-integer nonconvex problem due to the strong coupling among decision variables. Conventional methods become intractable in large-scale settings. Although deep reinforcement learning (DRL) has emerged as a promising approach~\cite{Zhou_2023_energy,Lin_2023,Fan_2024,Yan_2024}, existing DRL methods still struggle to handle coupled decisions and the exponentially large action space as network size increases.

\subsection{Related Work}\label{sec_related_work}
Most existing works focus on joint optimization of service placement and task offloading or task migration in edge/cloud computing networks. Du \textit{et al.}~\cite{Du_2024} proposed a framework that jointly learns service placement and task offloading to minimize task completion latency under user mobility in UAV-assisted MEC systems. Similarly, the authors in \cite{Chu_2024} introduced a two-timescale framework combining long-term service caching with short-term task offloading. Fan \textit{et al.}~\cite{Fan_2024_TMC} developed a collaborative approach for service placement and task scheduling across neighboring edge nodes, while a joint service caching, offloading, and migration in UAV-assisted networks was studied in \cite{Zhao_2025_TMC}. However, these works generally assume relatively stable traffic patterns and do not explicitly address load imbalance caused by highly dynamic IoT workloads. Moreover, the ES cooperation combined with service placement, user association, and computational delegation remains an open problem.

To explicitly capture load imbalance, queuing-theoretic models have been widely adopted. For instance, Yang \textit{et al.}~\cite{Yang_2023} modeled fixed-service queues to enable task migration between heavily and lightly loaded ESs. Chen \textit{et al.}~\cite{Chen_2022} and Hu \textit{et al.}~\cite{Hu_2025} proposed queuing-based models to estimate service delays under different orchestration strategies. More recently, the multi-instance M/M/c queuing model was considered in \cite{Wang_2025}  for the joint service deployment and request routing. However, these approaches typically rely on steady-state assumptions and fixed service configurations, which may not hold under highly dynamic and non-uniform IoT traffic conditions.

In recent years, DRL has emerged as a powerful tool for edge/cloud optimization due to its ability to handle high-dimensional and dynamic decision-making problems. Existing studies have explored various action space designs to address the complexity of joint optimization. For example, Zhou \textit{et al.}~\cite{Zhou_2023_energy} adopted a combinatorial action space for joint offloading and service caching, while Lin \textit{et al.}~\cite{Lin_2023} proposed reduced action representations to improve scalability in vehicular edge computing. In \cite{Fan_2024}, a hybrid action space integrating discrete and continuous variables was developed to minimize system cost by accounting for both time and energy consumption. Other approaches \cite{hubert2021learning, majeed2021exact, zhu2021overview} employ decomposition or encoder–decoder architectures to mitigate the impact of large action spaces.
Despite these efforts, existing DRL-based solutions often fail to fully capture the strong coupling among decision variables in joint optimization problems. As a result, the action space remains large and complex, leading to slow convergence and limited scalability in large-scale edge–cloud systems.

\begin{table*}[t]
\centering
\caption{Comparison of Key Features Between Our Work and Existing Approaches}
\label{tab:comparison}
\scriptsize
\begin{tabular}{c|c|c|c|c|c|c|c}
\hline
\textbf{Schemes} 
& \textbf{ES/ES}  
& \textbf{ES/CS}
& \textbf{User} 
& \textbf{Service} 
& \textbf{Task offloading,} 
& \textbf{Multi-layer DRL} 
& \textbf{Task Queuing} \\
& \textbf{cooperation}
& \textbf{cooperation}
& \textbf{Association}
& \textbf{Placement}
& \textbf{UE power Allocation}
& \textbf{with Latent Action Space}
& \textbf{Stability} \\
\hline
{\cite{Yang_2023}}  & $\checkmark$ & $\times$ & $\times$ & $\times$ & $\times$ & $\times$ & $\checkmark$  \\
{\cite{Chu_2024}} & $\checkmark$  & $\times$ & $\checkmark$ & $\checkmark$ & $\checkmark$ & $\times$ & $\times$   \\
{\cite{Zhou_2023_energy}} & $\checkmark$  & $\checkmark$ & $\times$ & $\checkmark$ & $\checkmark$ & $\times$ & $\times$   \\
{\cite{Fan_2024}} & $\checkmark$  & $\times$ & $\checkmark$ & $\checkmark$ & $\checkmark$ & $\times$ & $\times$  \\
{\cite{Fan_2024_TMC}} & $\checkmark$  & $\checkmark$ & $\times$ & $\checkmark$ & $\checkmark$ & $\times$ & $\checkmark$  \\
{\cite{Wang_2025}}& $\checkmark$  & $\times$ & $\times$ & $\checkmark$ & $\times$ & $\times$ & $\checkmark$  \\
\textbf{Our work}
      & $\checkmark$  & $\checkmark$ & $\checkmark$ & $\checkmark$ & $\checkmark$ & $\checkmark$ & $\checkmark$\\
\hline
\end{tabular}
\end{table*}

 \subsection{Motivation and Main Contributions}\label{sec_motivation_and_contributions}
Although HECC can significantly reduce latency in IoT networks, its performance is limited by several key challenges in dynamic and large-scale environments. First, dynamic task arrivals and heterogeneous resources lead to load imbalance across ESs and between ES and CS layers, resulting in severe queuing delays. While queuing-based methods~\cite{Yang_2023, Wang_2025} and collaborative frameworks~\cite{Fan_2024_TMC} partially address this issue, they often rely on fixed configurations or fail to fully capture dynamic multi-layer interactions. Second, existing approaches that optimize service placement and task offloading without computational delegation may create processing bottlenecks in multi-layer systems~\cite{huynh_service_placement,Sun_2024}, highlighting the need for joint optimization across system layers. Third, although DRL has been applied to such problems, the strong coupling among decision variables leads to extremely large action spaces, limiting scalability and convergence performance in large-scale networks~\cite{Zhou_2023_energy,Lin_2023,Fan_2024,Yan_2024}. Overall, these challenges, including load imbalance, incomplete multi-layer coordination, and DRL scalability limitations, remain insufficiently addressed, motivating the need for a scalable and adaptive optimization framework. Table~\ref{tab:comparison} summarizes key differences between our work and prior studies.

In this work, we consider an HECC system for large-scale IoT networks and propose a unified framework for jointly optimizing service placement, user association, computational delegation, task offloading, and user power allocation. Unlike existing approaches that rely on fixed service placement~\cite{Zhou_2023,Chu_2024,Shang_2024,Yao_2022}, the proposed framework enables dynamic adaptation to load imbalance, heterogeneous resources, and time-varying network conditions. In addition, proximal policy optimization (PPO) enables stable policy updates and naturally supports continuous control, making it suitable for short-term resource allocation. The main contributions are summarized as follows:
\begin{itemize}
\item We develop a computational delegation model for HECC-enabled IoT networks that jointly captures service placement, user association, ES–ES and ES–CS cooperation. The proposed formulation enables coordinated resource management across user, edge, and cloud layers and explicitly addresses load imbalance under resource, cost, and quality-of-service (QoS) constraints. We then formulate the joint service placement, computational delegation, and power control (JSCP) problem as an NP-hard mixed-integer nonconvex optimization problem.

\item To address the intractability of the JSCP problem, we exploit the inherent difference in decision dynamics and decompose it into long-term system configuration and short-term resource allocation subproblems. For the short-term subproblem, we employ PPO due to its stability and effectiveness in continuous control. For the long-term subproblem with a large combinatorial discrete action space, we integrate PPO with a latent action representation via a variational autoencoder (VAE), combining with the mapping table, enabling efficient exploration and scalable learning in large-scale networks. This results in a two-timescale multi-layer DRL framework (2T-MDRL-LA) that jointly handles discrete and continuous decisions.

\item Extensive simulations demonstrate that the proposed approach consistently outperforms benchmark schemes by reducing average e2e latency, accelerating convergence, and improving task offloading efficiency and resource utilization, while achieving near-optimal performance compared to branch-and-bound solutions.
In particular, it achieves up to a $20.8\%$ reduction in average latency and a $13\%$ improvement in resource utilization compared to the w/o computational delegation optimization (CDO) scheme, and converges approximately $50\%$ faster than conventional PPO.
\end{itemize}

\subsection{Paper Structure and Notations} \label{sec_paper_structure_notations}
The remainder of this paper is organized as follows. Section~\ref{sec_SystemModel} presents the system model, while Section~\ref{sec_Problem} formulates the problem and outlines the associated challenges. Section~\ref{sec_Solutions} introduces the proposed two-timescale multi-layer DRL framework. Section~\ref{sec_NumericalResults} reports the numerical results, and Section~\ref{sec_conclusion} concludes the paper.

\textbf{Notation:} Scalars are denoted by lower-case letters, vectors by boldface lower-case letters, and matrices by boldface upper-case letters. The absolute value of a complex scalar and the Euclidean norm of a vector are denoted by $\abs{\cdot}$ and $|\cdot|$, respectively. Moreover, $\mathcal{CN}(\mu,\sigma^2)$ denotes a circularly symmetric complex Gaussian distribution with mean $\mu$ and variance $\sigma^2$.

\section{System Model}\label{sec_SystemModel}
\begin{figure}[t]
	\centering
	\includegraphics[width=0.92\columnwidth,trim={0cm 0.0cm 0cm 0.0cm}]{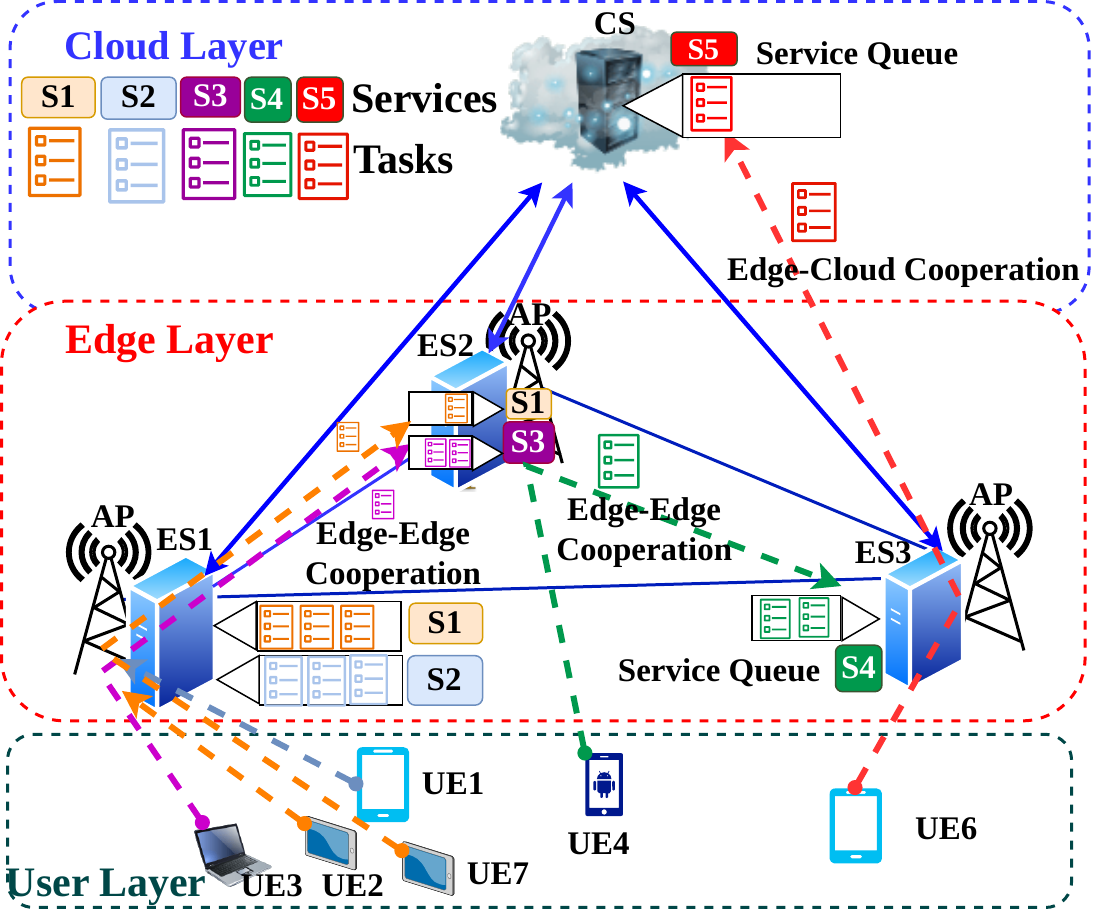}
	\caption{An HECC-enabled IoT network illustrating user, edge, and cloud layers.}
	\label{fig:system_model}
\end{figure}

The considered HECC-enabled IoT network consists of user, edge, and cloud layers, as illustrated in Fig.~\ref{fig:system_model}. Let $\mathcal{M} = \{1,2,\ldots,M\}$ denote the set of UEs, which are randomly distributed over the network coverage area. The edge layer comprises $K$ ESs, represented by $\mathcal{K} = \{1,2,\ldots,K\}$. Each access point (AP), co-located with an ES, is equipped with $L$ antennas to communicate with single-antenna UEs via wireless links. The ESs are interconnected with each other and the cloud server through fronthaul and backhaul links, respectively. The set of services is denoted by $\mathcal{N} = \{1,2,\ldots,N\}$.

The system operates over two timescales, as shown in Fig.~\ref{fig:twotime}: a long-term time-slot $t \in \mathcal{T} = \{1,2,\ldots,T\}$ and a short-term time-slot $t_s \in \mathcal{S} = \{t_1,t_2,\ldots,t_S\}$ within each long-term slot, where $T$ and $S$ denote the numbers of long-term and short-term slots, respectively. The duration of each short-term slot is $\Delta = \ell / S$, where $\ell$ is the duration of a long-term slot. At each long-term slot, service placement, user association, and computational delegation are optimized to maintain system stability and minimize average e2e latency. At each short-term slot, task offloading and user transmit power are optimized to adapt to time-varying wireless channel conditions.

\begin{figure}[t]
	\centering
	\includegraphics[width=0.95\columnwidth,trim={0cm 0.0cm 0cm 0.0cm}]{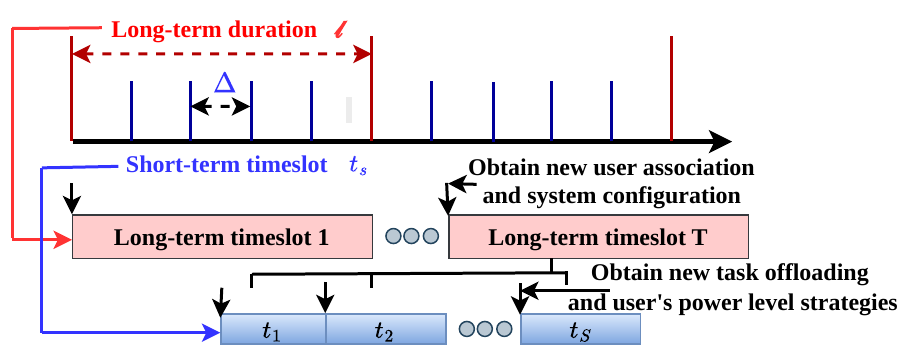}
	\caption{\small Two-timescale operation with long- and short-term decisions.}
	\label{fig:twotime}
\end{figure}

The working of the three layers in this system can be illustrated as follows:

\noindent\textbf{1) User layer:} Each UE $m \in \mathcal{M}$ requests at most one service during each long-term time slot. At each short-term time-slot $t_s$, UE $m$ generates a computation task $J_m[t_s] \triangleq (\tau_m^{\max}[t_s], D_m[t_s], C_m[t_s])$ associated with a service $n \in \mathcal{N}$. Task arrivals follow a Poisson process with mean rate $\lambda_m^{\mathtt{ue}}$ (tasks/s). Each task can either be executed locally with processing rate $f_m$ (cycles/s) or offloaded to the edge layer. Here, $\tau_m^{\max}[t_s]$, $D_m[t_s]$, and $C_m[t_s]$ denote the task deadline, task size, and required CPU cycles, respectively. We define $b_m^k[t] \in \{0,1\}$ to indicate whether UE $m$ is associated with ES $k$ at long-term slot $t$. Specifically, $b_m^k[t]=1$ if UE $m$ connects to ES $k$ for task offloading, and $0$ otherwise. Each UE can be associated with at most one ES, i.e., $\sum_{k \in \mathcal{K}} b_m^k[t] \le 1$. The offloading decision at short-term slot $t_s$ is denoted by $\beta_m[t_s] \in \{0,1\}$, where $\beta_m[t_s]=1$ indicates offloading and $\beta_m[t_s]=0$ indicates local execution.

\noindent\textbf{2) Edge layer:} Let $b_n^k[t] \in \{0,1\}$ indicate whether service $n \in \mathcal{N}$ is installed on ES $k$ at time-slot $t$, where $b_n^k[t]=1$ if service $n$ is deployed on ES $k$, and $0$ otherwise. Due to resource constraints, each ES $k$ must host at least one service and can support at most $N_k^{\max}$ services, \textit{i.e.}, $1 \le \sum_{n \in \mathcal{N}} b_n^k[t] \le N_k^{\max}, \forall k \in \mathcal{K}$.
We define $b_m^{k,k'}[t] \in \{0,1\}$ to indicate whether ES $k$ processes or delegates the task of UE $m$ to ES $k'$ for service execution at time-slot $t$. Specifically, $b_m^{k,k}[t]=1$ indicates local execution at ES $k$, while $b_m^{k,k'}[t]=1$ with $k' \neq k$ indicates that the task is offloaded to another ES $k'$; otherwise, $b_m^{k,k'}[t]=0$.

If edge resources are insufficient, ES $k$ may forward the admitted task to the CS. Accordingly, we define the cloud-forwarding decision variable $b_m^{k,\cs}[t] \in \{0,1\}$, where $b_m^{k,\cs}[t]=1$ indicates that the task of UE $m$ is forwarded from ES $k$ to the cloud, and $0$ otherwise. Each admitted task is executed either locally, offloaded to another ES, or forwarded to the cloud, yielding the constraint: $\sum_{k' \in \mathcal{K}} b_m^{k,k'}[t] + b_m^{k,\cs}[t] \le 1$.

\noindent\textbf{3) Cloud layer:} The cloud server hosts all services with a maximum processing capacity $f_{\max}^{\cs}$. Tasks offloaded to the cloud are processed with rate $f^{\cs} = f_{\max}^{\cs}/N$ (cycles/s).

\subsection{Multi-layer Transmission Latency Model}\label{sec_transmission_model}
\textbf{Wireless transmission latency:} We consider a frequency-division multiple access (FDMA)-based uplink transmission, where equal bandwidth $w$ is allocated to each UE, \textit{i.e.} $w = W/M$, with $W$ (MHz) denoting the total system bandwidth. The channel vector between UE $m$ and AP $k$ at short-term time-slot $t_s$ is modeled as $\boldsymbol{h}_m^k[t_s] = \sqrt{g_m^k[t]},\bar{\boldsymbol{h}}_m^k[t_s]$, where $g_m^k[t]$ represents the large-scale channel gain which remains constant during a long-term slot, and $\bar{\boldsymbol{h}}_m^k[t_s] \sim \mathcal{CN}(0, \mathbf{I}_L)$ denotes small-scale fading.

The signal-to-noise ratio (SNR) of UE $m$ received at AP $k$ is given by
\begin{align}\label{eq:SNR}
 \gamma _m^k[t_s] = \frac{{{p_m[t_s]}{{\big\| {\boldsymbol{h}_m^k\left[ {{t_{s}}} \right]} \big\|}^2}}}{{w{N_0}}}
 \end{align}
where $p_m[t_s]$ and $N_0$ denote the transmit power of UE $m$ and the noise power density, respectively. Accordingly, the uplink data rate of UE $m$ at AP $k$ is expressed as
\begin{align}\label{eq:Rm}
   R_m[t_s] = \sum\limits_{k \in \mathcal{K}} b_{m}^k[t] \frac{w}{{\ln 2}}\ln \left( 1 + \gamma_m^k[t_s]\right).  
\end{align}
Given a task of size $D_m[t_s]$, the uplink transmission latency for offloading it from UE $m$ to ES $k$ is  given by
\begin{align}\label{eq:T_mk_radio}
\tau_{m}^k[t_s] = \frac{\beta_{m}[t_s]D_{m}[t_s]}{{ R_m[t_s]}}.
\end{align}
The downlink transmission delay can be neglected in this work due to the relatively small size of computation results and the high transmit power at APs~\cite{latency_2020}.

\textbf{Fronthaul transmission latency: } When the required service is unavailable at ESs or edge resources are insufficient, tasks are forwarded from ES $k$ to $k'$. The worst-case fronthaul transmission latency is modeled as
\begin{align}\label{eq:T_kk_ts}
   \tau _{m}^{k,k'}\left[ {{t_s}} \right] = \frac{\sum\limits_{m \in \mathcal{M}}  {\! \! b_m^k\left[ t \right] \! \! \! \! \sum\limits_{k' \in \mathcal{K} \setminus \{k\}}{ \! \! \! \! b_m^{k,k'}\left[ t \right]b_n^{k'}[t]{\beta _m}\left[ {{t_s}} \right]{D_m}\left[ {{t_s}} \right]}}}{R_k^{k'}} 
\end{align}
where $R_k^{k'}$ denotes the fronthaul transmission rate between ES $k$ and ES $k'$.

\textbf{Backhaul transmission latency:}  
When service $n \in \mathcal{N}$ is not deployed at any ES or the ESs are fully occupied, tasks generated by users $m \in \mathcal{M}$ are forwarded from ES $k$ to the CS for processing. To characterize the worst-case backhaul bottleneck caused by simultaneous task forwarding, the transmission delay from ES $k$ to the CS is given by
\begin{align}\label{eq:T_kc_ts_bottleneck}
   \tau _{m}^{k,\cs}[{t_s}] = \frac{\sum\limits_{m \in \mathcal{M}} {b_m^{k,\cs}\left[ t \right]b_m^k\left[ t \right]{\beta _m}\left[ {{t_s}} \right]{D_m}\left[ {{t_s}} \right]}}{R_k^{\cs}}
\end{align}
where $R_k^{\cs}$ denotes the backhaul transmission rate.

Combining \eqref{eq:T_mk_radio}-\eqref{eq:T_kc_ts_bottleneck}, the total transmission latency of an offloaded task from UE $m$ is
\begin{align}\label{T_tot_ts}
  &\tau_{m}^{\mathtt{ts}}[t_s] = \tau_{m}^k[t_s] + \max_{\forall k} \bigl\{ \tau_{m}^{k,\mathtt{cs}}[t_s] 
  + \tau_{m}^{k,k'}[t_s] \bigr\}.
\end{align}

\subsection{Processing Latency Model}\label{sec_computation_model}
\textbf{The local processing}: The remaining portion of the task processed locally at UE $m$ is ${(1-\beta _{m}}\left[ {{t_s}} \right]){C_{m}}\left[ {{t_s}} \right]$. The local processing time at UE $m$ with a computation rate $f_m$ is given by
\begin{align}\label{eq:T_ue_cp}
    \tau_{m}^{\mathtt{ue}}[{t_s}] = \frac{{{(1-\beta _{m}}\left[ {{t_s}} \right]){C_{m}}\left[ {{t_s}} \right]}}{{{f_m}}}.
\end{align}

\textbf{Edge processing latency:} We assume that each service $n$ maintains a dedicated computation queue at ES $k$. The queue of service $n$ at ES $k$ at time $t_s$ is denoted by $\boldsymbol{q}_{n}^{k}(t_s) = \{q_{n, 1}^{k}(t_s), q_{n, 2}^{k}(t_s), \ldots, q_{n, I_{n}^k[t_s]}^{k}(t_s)\}$, where $I_n^k[t_s]$ represents the number of tasks in the queue at time $t_s$.

Each queue is allocated a dynamically adjusted computation capacity to improve resource utilization at the ES, given by $f_n^k[t] = \frac{f_{\max}^k}{\sum_{n \in \mathcal{N}} b_n^k[t]}$, where $f_{\max}^k$ denotes the maximum computational capacity of ES $k$. Let $i_m[t_s]$ denote the position of task $J_m[t_s]$ in the queue. The queuing delay of task $J_m[t_s]$ at ES $k$ is given by
$\tau_{m}^{\mathtt{es,q}}[t_s] = \sum_{i = 1}^{i_m[t_s] - 1} \frac{\mathcal{C}\big(q_{n,i}^k(t_s)\big)}{f_n^k[t]},$
where $\mathcal{C}\big(q_{n,i}^k(t_s)\big)$ denotes the required number of computational cycles of the task at position $i$ in the queue. The processing time of task $J_{m}[t_s]$ at ES $k$ is then expressed as
\begin{align} \label{eq:T_c_ES}
    \tau_{m}^{k,\cp}[t_s] = \tau_{m}^{\mathtt{es,q}}[t_s] + \frac{\beta_{m}[t_s]C_{m}[t_s]}{f_n^k[t]}.
\end{align}
The queue length of service $n$ at ES $k$ evolves according to $I_{n}^k[t_s+1] = I_{n}^k[t_s] - I_n^{\mathsf{out}}[t_s] + I_n^{\mathsf{in}}[t_s+1]$, where $I_{n}^k[t_s]$, $I_n^{\mathsf{out}}[t_s]$ and $I_n^{\mathsf{in}}[t_s + 1]$ denote the number of queued tasks, completed tasks, and newly arrived tasks at the corresponding time slots, respectively \cite{Deng_2021}.

\textbf{Cloud processing latency:} The queue of service $n$ at  CS at time $t_s$ is denoted by $\boldsymbol{q}_{n}^{\cs}(t_s) = \{q_{n, 1}^{\cs}(t_s), q_{n, 2}^{\cs}(t_s), \cdot \cdot \cdot, q_{n, I_{\mathtt{n}}^{\cs}[t_s]}^{k}(t_s)\}$ where $I_n^{\cs}[t_s]$ denotes the number of tasks in the queue and is updated in the same manner as $I^k_n[t_s]$. The queuing delay of task $J_m[t_s]$ at the CS is given by
$\tau_{m}^{\mathtt{cs,q}}[t_s]=\sum_{i = 1}^{i_m - 1} \mathcal{C}\big(q_{n,i}^{\cs}(t_s)\big) / f^{\cs}$, where $\mathcal{C}\big(q_{n,i}^{\cs}(t_s)\big)$ denotes the required number of computational cycles of the task at position $i$. The processing time of task $J_m[t_s]$ at  CS is given as
\begin{align}\label{eq:T_c_cp}
    \tau_{m}^{\text{cs,cp}}[t_s] = \tau_{m}^{\text{cs,q}}[t_s] + \frac{\beta_{m}[t_s]C_{m}[t_s]}{f^{\cs}}.
\end{align} 

From \eqref{eq:T_ue_cp}-\eqref{eq:T_c_cp}, the total processing time of the task associated with service $n$ from UE $m$ is given by
\begin{align}
\tau_{m}^{\mathtt{cp}}[t_s] &= \tau_{m}^{\mathtt{ue}}[t_s] + \mathop {\max }\limits_{\forall k} \Big \{  b_{m}^k\left[ t \right]b_n^k[t]  \tau_{m}^{k,cp}[t_s] \nonumber \\
&\quad \times \Big(1 -  {\sum_{k' \in \mathcal{K} \setminus \{k\}} {b_{m}^{k,k'}\left[ t \right] - b_{m}^{k,\mathtt{cs}}\left[ t \right]} } \Big) \nonumber \\
&\quad 
+ b_{m}^k[t] \!\!\sum_{k' \in \mathcal{K} \setminus \{k\}} \!\!b_{m}^{k,k'}[t] b_n^{k'}[t] \tau_{m}^{k',cp}[t_s]  \nonumber \\
&\quad + b_{m}^k[t] b_{m}^{k,\mathtt{cs}}[t]\tau_{m}^{\mathtt{cs,cp}}[t_s] \Big\}. 
\label{T_tot_cp}
\end{align}

\subsection{Monetary Cost Model}\label{sec_monetary_cost_model}
\textbf{Service placement Cost:} Installing or uninstalling a service at ESs incurs resource consumption, such as CPU cycles and memory~\cite{Jiaming_2024}. Let $\rho^{\mathtt{i}}$ and $\rho^{\mathtt{u}}$ ($\$/\text{service}$) denote the costs of service installation and uninstallation at ESs, respectively, where $\rho^{\mathtt{u}} \leq \rho^{\mathtt{i}}$. Define $\Omega_n^k[t] = b_n^k[t] - b_n^k[t-1]$ as the change in the deployment status of service $n$ at ES $k$ from time frame $t-1$ to $t$. Accordingly, $\Omega_n^k[t] \in \{-1, 0, 1\}$ represents service uninstallation, no change, and installation, respectively. We assume that all services are initially cached at ESs, and therefore, no latency is incurred during service activation or deactivation.

Let $c_n^k[t]$ denote the service placement cost of service $n$ at ES $k$ during time frame $t$. Then, the cost is given by
\begin{align}\label{eq:deploy_cost}
c_n^k[t]
 = \frac{1}{2} {\left( \Omega_n^k[t] \right)}^2 \left( {\rho ^\mathtt{u}} + {\rho ^\mathtt{i}} \right) - \frac{1}{2}\Omega_n^k[t]\left( {\rho ^\mathtt{u}} - {\rho ^\mathtt{i}} \right).
\end{align}

\textbf{Operation and computational cost:} Operating service $n$ at ES $k$ consumes edge resources such as memory and computational capacity. Let $\rho^{\mathtt{o}}$ ($\$/\text{service}$) denote the operation cost of a service at an ES. Accordingly, the operation cost of service $n$ at ES $k$ is given by $b_n^k[t]\rho^{\mathtt{o}}$. In addition, the cloud server (CS) incurs resource consumption when processing offloaded tasks~\cite{wang2021eihdp_price,Li_twotime_cost,Thinh_physycal_cost,Yang_cost_equation}. Let $\rho^{\mathtt{cs}}$ ($\$/\text{request}$) denote the processing cost per task offloaded from an ES to the CS. The corresponding computational cost is given by $b_{m}^{k,\mathtt{cs}}[t]\rho^{\mathtt{cs}}$. Therefore, the total system cost $c^{\mathtt{tot}}[t]$ at long-term time frame $t$, including service placement cost, service operation cost at ESs, and computational cost at the CS, is expressed as
\begin{align}
{c^{\mathtt{tot}}}[t] =& \sum_{k \in \mathcal{K}} \sum_{n \in \mathcal{N}} \left( {c_n^k[t] + b_n^k[t]{\rho ^\mathtt{o}}} \right)  \nonumber\\
& + \sum_{k \in \mathcal{K}} \sum_{m \in \mathcal{M}} {b_{m}^{k,\mathtt{cs}}[t]} {\rho ^{\mathtt{cs}}}.
\end{align}

\subsection{Energy Consumption Model of Users}\label{sec_energy_consumption_model}

Let $e_{m}^{\mathtt{cp}}[t_s]$ and $e_{m}^{\mathtt{ts}}[t_s]$ denote the energy consumption of UE $m$ for local computation and wireless transmission at time-slot $t_s$, respectively. The total energy consumption is given by~\cite{ling2021}
\begin{align}
 e_{m}[t_s] = &e_{m}^{\mathtt{cp}}[t_s] + e_{m}^{\mathtt{ts}}[t_s] = \frac{{{\mu _m}}}{2}(1-{\beta _{m}}\left[ {{t_s}} \right]){ C_{m}}\left[ {{t_s}} \right]{f_m^2} \nonumber \\
 &+ {p_m[t_s]}\frac{{ {\beta_{m}}\left[ {{t_s}} \right] {D_{m}}\left[ {{t_s}} \right]}}{{{R_m}[t_s]} }
\end{align}
where $\mu_m/2$ (in $\text{W}\cdot\text{s}^3/\text{cycle}^3$) denotes the effective switched-capacitance coefficient of UE $m$.

\section{Problem Design and Analysis}\label{sec_Problem}
\subsection{The Problem Design}\label{sec_Problem_main}
The overall e2e latency of task $J_m[t_s]$ for UE $m$ consists of both transmission latency and computation (processing) latency. Specifically, by combining the transmission latency  in \eqref{T_tot_ts} (\textit{e.g.}, from UE $m$ to APs and between APs and CS) and the processing latency in \eqref{T_tot_cp} (\textit{e.g.}, at the UE, ES, or CS), the total e2e latency is given by:
\begin{align}
    \tau_{m}^{\mathtt{e2e}}[t_s] = \tau_{m}^{\mathtt{ts}} [t_s] +\tau_{m}^{\mathtt{cp}}[t_s].
\end{align}
The objective is to minimize the average e2e latency across all UEs, defined as
\begin{equation}
\begin{aligned}
\xi( {{\boldsymbol{b}[t]},\boldsymbol{\beta}[t_s],\boldsymbol{p}\left[ {{t_{s}}} \right]})= \frac{1}{MS} \sum_{m \in \mathcal{M}} { \sum_{t_s \in \mathcal{S}} {{\tau_{m}^{\mathtt{e2e}}\left[ {{t_s}} \right]} }}
\end{aligned}
\end{equation}
where $\boldsymbol{b}[t] \triangleq \big\{b_n^k[t], b_m^k[t], b_m^{k,k'}[t], b_m^{k,\cs}[t]\big\}_{\forall m,n,k}$, $\boldsymbol{\beta}[t_s]\triangleq\{\beta_{m}[t_s]\}_{\forall m}$, and 
$\boldsymbol{p}\left[ {{t_{s}}} \right]\triangleq \{p_m\left[ t_{s} \right]\}_{\forall m}$.

In this paper, we investigate the joint service placement, computational delegation offloading, and user transmit power control (JSCP) problem.  The JSCP problem is formulated as
\begin{subequations} \label{eq:Main}
	\begin{IEEEeqnarray}{cl}
		 \text{JSCP:}\,&\mathop {\min }\limits_{{\boldsymbol{b}},\boldsymbol{\beta} ,\boldsymbol{p}}\, \xi \left( {{\boldsymbol{b}[t]},\boldsymbol{\beta}[t_s],\boldsymbol{p}\left[ {{t_{s}}} \right]} \right)\label{eq:Maina} \\
		& \st  \ \tau_{m}^{\mathtt{e2e}}\left[ {{t_{s}}} \right] \le \tau_m^{\max },\,\forall m                       \label{eq:Mainb}\\
		&\qquad e_{m}[t_s] \le e_m^{\max },\,\forall m \label{eq:Mainc} \\
        &\qquad  R_m[t_s] \ge {R_{\min }},\,\forall m\label{eq:Maind} \\
        &\qquad c^{\mathtt{tot}}[t] \le c^{\mathtt{max}} \label{eq:Maine} \\
        &\qquad \boldsymbol{b}[t] \in \mathscr{B}[t], \boldsymbol{\beta}[t_s] \in \ \boldsymbol{\psi}[t_s], \boldsymbol{p}[t_s] \in \mathscr{P}[t_s] \qquad\quad\label{eq:Mainf} 
	\end{IEEEeqnarray}
\end{subequations}
where $\mathscr{B}[t], \boldsymbol{\psi}[t_s]$ and $\mathscr{P}[t_s]$ denote the feasible sets, defined as follows:
\begin{subequations}\label{setb}
   \begin{IEEEeqnarray}{cl}
    \mathscr{B}[t] \buildrel \Delta \over = 
    &\Big\{ \boldsymbol{b}[t]\Bigl| \,\boldsymbol{b}[t] \in \{0,1\},\,\label{setb_a} \\
    &\quad\sum_{k \in \mathcal{K}} b_{m}^k[t] \le 1, \forall m, \label{setb_b} \\
    &\quad \sum_{k' \in \mathcal{K} \setminus \{k\}} b_{m}^{k,k'}[t] + b_{m}^{k,\cs}[t] \le 1, \forall k,m,\label{setb_c} \\
    &\quad \sum_{k' \in \mathcal{K} \setminus \{k\}} b_{m}^{k,k'}[t] \le 1, \forall m,k,\qquad \label{setb_d}\\ 
      &\quad 1 \le \sum_{n \in \mathcal{N}} b_n^k[t] \le N_k^{\max}, \forall k\Big\} \label{setb_e}
   \end{IEEEeqnarray} 
\end{subequations}
\begin{align}
   \boldsymbol{\psi}[t_s] \buildrel \Delta \over =  \Big \{\boldsymbol{\beta}[t_s] \Bigl| \, \beta_{m}[t_s] \in \{0,1\}, \forall m \Big\}\label{setphi}
\end{align}
\begin{align}
    \mathscr{P}[t_s] \buildrel \Delta \over = \Big \{ \boldsymbol{p}[t_s] \Bigl| \, 0 \le p_m[t_s] \le P_m^{\mathtt{max}}, \forall m \Big\}.\label{setp}
\end{align}
In problem~\eqref{eq:Main}, constraints~\eqref{eq:Mainb} and~\eqref{eq:Mainc}  enforce the task deadline and UE energy consumption limits, respectively. Constraints~\eqref{eq:Maind} and~\eqref{eq:Maine} enforce the minimum transmission-rate requirement and the system cost budget over each long-term time frame. Constraint~\eqref{eq:Mainf} specifies the feasible sets of the optimization variables, which are detailed in~\eqref{setb}-\eqref{setp}. In particular, constraint~\eqref{setb} captures user association, ES–ES cooperation, ES–CS
cooperation, and service placement, as detailed in Section~\ref{sec_SystemModel}. Constraints~\eqref{setphi} and~\eqref{setp} define the feasible sets for offloading decisions, and the UE transmit power, respectively.

\textit{Challenges of Solving Problem JSCP:} The JSCP problem is NP-hard and computationally intractable using conventional optimization methods. Specifically, the objective \eqref{eq:Maina} and constraints \eqref{eq:Mainb}-\eqref{eq:Maind} are nonconvex, resulting in a mixed-integer nonlinear programming (MINLP) problem. The strong coupling between binary variables $(\boldsymbol{b}[t], \boldsymbol{\beta}[t_s])$ and continuous variables $(\boldsymbol{p}[t_s])$ introduces combinatorial complexity while preventing convex reformulation. Moreover, the presence of multiple binary variables significantly enlarges the solution space, leading to exponential growth in computational complexity. As a result, even advanced optimization techniques, such as the improved branch-and-bound (BnB) algorithm (IBBA)~\cite{Vu_2021}, become computationally prohibitive for solving JSCP directly, since the nonconvexity remains even after relaxing the integer constraints.

\subsection{Tractable Reformulation of \eqref{eq:Main} }\label{subsec_Transformation}
To begin with, we introduce structural lemmas that characterize the relationships among binary decision variables, enabling the decomposition of their strong coupling for a more tractable reformulation.

 \begin{lemma}[Service Deployment Availability] \label{LM_lemma1} If service $n$ is not deployed at ES $k$, any offloaded task associated with service $n$ cannot be processed locally at ES $k$ and must instead be forwarded to another ES $k'$ or to the cloud. This condition is expressed as
\begin{align}\label{lemma1}
    b_{m}^k\left[ t \right] - \Big( {\sum_{k' \in \mathcal{K} \setminus \{k\}} {b_{m}^{k,k'}\left[ t \right] + b_{m}^{k,\mathtt{cs}}\left[ t \right]} } \Big) \le b_n^k\left[ t \right].
\end{align}
\end{lemma}

\begin{lemma}[Association-Constrained Task Delegation]\label{LM_lemma2} An ES $k$ can delegate an offloaded task from UE $m$ to another ES $k'$ or to CS only if UE $m$ is associated with ES $k$ at time-slot $t$, which are expressed as
\begin{subequations}\label{lemma2}
    \begin{IEEEeqnarray}{rcl}
        \sum_{k' \in \mathcal{K} \setminus \{k\}} {b_{m}^{k,k'}\left[ t \right] + b_{m}^{k,\mathtt{cs}}\left[ t \right]} &\le b_{m}^k\left[ t \right], \forall m,k\label{lemma2.1} \\
         \sum_{k' \in \mathcal{K} \setminus \{k\}} {b_{m}^{k,k'}\left[ t \right]\,}  &\le b_{m}^k\left[ t \right],\forall m,k \label{lemma2.2}\\
          b_{m}^{k,\mathtt{cs}}\left[ t \right] &\le b_{m}^k\left[ t \right],\forall m,k.\label{lemma2.3}
    \end{IEEEeqnarray}
\end{subequations}
\end{lemma}
\begin{lemma}[Service-Constrained Assistant ES Condition]\label{LM_lemma3} If UE $m$ is associated with ES $k$ and offloads a task associated with service $n$, then forwarding the task from ES $k$ to another ES $k'$ is feasible only if service $n$ is deployed at ES $k'$ at time-slot $t$, which is expressed as
    \begin{align}\label{lemma3}
    b_{m}^k[t] b_{m}^{k,k'}\left[ t \right] \le b_n^{k'}\left[ t \right],\, \forall k', m. 
\end{align}
\end{lemma}

\begin{proof}
Consider $x, y \in \{0,1\}$ with $x \leq y$. Then, if $y = 0$, it must hold that $x = 0$; if $y = 1$, $x$ can take either value in $\{0,1\}$.   For Lemma~\ref{LM_lemma1}, consider the service deployment indicator $b_n^k[t]$. If $b_n^k[t] = 0$, service $n$ is not deployed at ES $k$, and thus any task associated with service $n$ cannot be processed locally at ES $k$. In this case, if $b_m^k[t] = 1$ (\textit{i.e.} the task is admitted by ES $k$), the task must be forwarded to another ES $k'$ or to CS, implying that 
 $\sum_{k' \in \mathcal{K} \setminus \{k\}} b_m^{k,k'}[t] + b_m^{k,\mathtt{cs}}[t] = 1$. Hence, the inequality in \eqref{lemma1} holds.

If $b_n^k[t] = 1$, service $n$ is available at ES $k$. In this case, if $b_m^k[t] = 1$, the task can either be processed locally or forwarded. Therefore, 
$\sum_{k' \in \mathcal{K} \setminus \{k\}} b_m^{k,k'}[t] + b_m^{k,\mathtt{cs}}[t] \in \{0,1\}$, 
and the inequality in \eqref{lemma1} is still satisfied. This completes the proof of Lemma~\ref{LM_lemma1}.

Lemmas~\ref{LM_lemma2} and~\ref{LM_lemma3} follow from the same binary control relationship, where the forwarding variables are constrained by the association variable $b_m^k[t]$, ensuring valid task delegation.
\end{proof}

\begin{proposition}\label{pro1}
By applying Lemmas~\ref{LM_lemma1}--\ref{LM_lemma3}, the strong coupling among the binary decision variables in \eqref{eq:T_kk_ts}, \eqref{eq:T_kc_ts_bottleneck}, and \eqref{T_tot_cp} can be effectively decomposed. Consequently, the latency expressions can be reformulated into the following equivalent and simplified forms:
    \begin{subequations} \label{eq_pro1}
        \begin{IEEEeqnarray}{rcl}
             && \tau _{m}^{k,k'}\left[ {{t_s}} \right] = \frac{\sum\limits_{m \in \mathcal{M}}  { \sum\limits_{k' \in \mathcal{K} \setminus \{k\}}{ \! \! \! \! b_m^{k,k'}\left[ t \right]{\beta _m}\left[ {{t_s}} \right]{D_m}\left[ {{t_s}} \right]}}}{R_k^{k'}} \label{eq:T_kk_ts_reform}\\
             &&  \tau _{m}^{k,\cs}[{t_s}] = \sum\limits_{m \in \mathcal{M}} {b_m^{k,\cs}\left[ t \right]{\beta _m}\left[ {{t_s}} \right]{D_m}\left[ {{t_s}} \right]} /R_k^{\cs} \label{eq:T_kc_ts_bottleneck_reform} \\
            &&  \tau_{m}^{\mathtt{cp}}[t_s] = \tau_{m}^{\mathtt{ue}}[t_s] + \mathop {\max }\limits_{\forall k} \Big \{ ( b_{m}^k[t] - \!\!\!\!\!\!\!\sum_{k' \in \mathcal{K} \setminus \{k\}} \!\!\!\!\!\!\! b_{m}^{k,k'}[t] - b_{m}^{k,\mathtt{cs}}[t] )   \nonumber \\ 
            && \times \tau_{m}^{k, \mathtt{cp}}[t_s]  + \!\!\!\!\!\!\!\sum_{k' \in \mathcal{K} \setminus \{k\}} \! \! \!\!\!\!b_{m}^{k,k'}[t] \tau_{m}^{k',\mathtt{cp}}[t_s] + b_{m}^{k,\mathtt{cs}}[t]\tau_{m}^{\mathtt{cs,cp}}[t_s] \Big\}.\quad\,\label{T_tot_cp_reform} 
        \end{IEEEeqnarray}
    \end{subequations}
\end{proposition}
\begin{proof}

Considering $x_1 \leq x_2$ and $x_1x_2 \leq x_3$, and following Table \ref{tab:true_table}, we have the output $y = x_1$, and then
\begin{align}\label{proof_logic_simplify}
    \left\{ {\begin{array}{*{20}{c}}
{y = x_1x_2x_3}\\
{x_1 \le x_2}\\
{x_1x_2 \le x_3}
\end{array}} \right. \Leftrightarrow \left\{ {\begin{array}{*{20}{c}}
{y = x_1}\\
{x_1 \le x_2}\\
{x_1x_2 \le x_3.}
\end{array}} \right.
\end{align}
\begin{table}[htb]
\centering  
\captionof{table}{Truth Table}
\label{tab:true_table}
    \begin{tabular}{@{}ccc|c@{}}
        \hline
        $x_1$ & $x_2$ & $x_3$& $y = x_1 x_2 x_3$ \\
        \hline\hline
       0 & 0 & 0 & 0 \\
       0 & 1 & 0 & 0 \\
       0 & 0 & 1 & 0 \\
       0 & 1 & 1 & 0 \\
       1 & 1 & 1 & 1 \\
        \hline                          
    \end{tabular}
\end{table}

Similarly, it follows that
\begin{align}\label{proof_logic_simplify_2}
    \left\{ {\begin{array}{*{20}{c}}
{y = x_1x_2}\\
{x_1 \le x_2}
\end{array}} \right. \Leftrightarrow \left\{ {\begin{array}{*{20}{c}}
{y = x_1}\\
{x_1 \le x_2.}
\end{array}} \right.
\end{align}
These equivalences eliminate redundant binary products. Applying Lemmas~\ref{LM_lemma1}--\ref{LM_lemma3} with \eqref{proof_logic_simplify} and \eqref{proof_logic_simplify_2} yields \eqref{eq_pro1}.

\textit{Derivation of~\eqref{eq:T_kk_ts_reform}}: From \eqref{lemma2.2} and \eqref{lemma3}, the product term can be reduced to $b_m^{k,k'}[t]$. Applying this to \eqref{eq:T_kk_ts}, we obtain
\begin{align}\label{proof_23a}
    \left\{ {\begin{array}{*{20}{c}}
{\!\!y = \!\!\!\sum\limits_{m \in \mathcal{M}}  \! \! b_m^k\left[ t \right] \! \! \! \!\!\!\! \sum\limits_{k' \in \mathcal{K} \setminus \{k\}} \! \! \! \! b_m^{k,k'}[t]b_n^{k'}[t]}\\
{~\eqref{lemma2.2}}, ~\eqref{lemma3}
\end{array}} \right. \!\!\!\!\!\!\!\!\!\Leftrightarrow \!\! \left\{ {\begin{array}{*{20}{c}}
{\!\!\!y = \!\!\!\!\!\sum\limits_{k' \in \mathcal{K} \setminus \{k\}} \! \! \! \!\!\! b_m^{k,k'}\![t]}\\
{~\eqref{lemma2.2}},~\eqref{lemma3}.
\end{array}} \right.
\end{align}

\textit{Derivation of~\eqref{eq:T_kc_ts_bottleneck_reform}}: From \eqref{lemma2.3} and \eqref{proof_logic_simplify_2}, the product $b_m^{k,\mathtt{cs}}[t] , b_m^k[t]$ reduces to $b_m^{k,\mathtt{cs}}[t]$. Substituting into \eqref{eq:T_kc_ts_bottleneck} yields
\begin{align}\label{proof_23b}
    \left\{ {\begin{array}{*{20}{c}}
{y = \!\!\!\sum\limits_{m \in \mathcal{M}}\!\!\! b_m^{k,\cs}\left[ t \right]b_m^k\left[ t \right]}\\
{\eqref{lemma2.3}}
\end{array}} \right. \!\!\!\Leftrightarrow \left\{ {\begin{array}{*{20}{c}}
{y = \!\!\!\sum\limits_{m \in \mathcal{M}}\!\!\! b_m^{k,\cs}\left[ t \right]}\\
{\eqref{lemma2.3}.}
\end{array}} \right.
\end{align}

\textit{Derivation of~\eqref{T_tot_cp_reform}}: Similarly, applying \eqref{proof_logic_simplify_2} together with \eqref{lemma1}, \eqref{lemma2.1}, and \eqref{proof_23a}–\eqref{proof_23b} decouples the binary terms in \eqref{T_tot_cp}, yielding \eqref{T_tot_cp_reform}..
\end{proof}

\subsection{Two-Timescale Decomposition of the JSCP Problem}
Problem~\eqref{eq:Main} is solved using a two-timescale formulation that separates long-term system configuration from short-term adaptive control. Specifically, long-term decisions $\boldsymbol{b}[t]$ determine user association and service configuration at time-slot $t$, capturing relatively stable system characteristics and ensuring system stability. In contrast, short-term decisions, including UE transmit power $\boldsymbol{p}[t_s]$ and task offloading $\boldsymbol{\beta}[t_s]$, are optimized at each timeslot $t_s$ to adapt to time-varying network conditions.

\textit{1) Long-term subproblem (L-SP)}: The long-term subproblem optimizes service placement, user association, and computation delegation at the ES and CS layers over the long-term time-slot $t$. Based on \eqref{eq:Main} and Lemmas~\ref{LM_lemma1}--\ref{LM_lemma3}, the L-SP with respect to $\boldsymbol{b}[t]$ is formulated as
 \begin{subequations} \label{eq:long-term}
    \begin{IEEEeqnarray}{cl}
         \textbf{L-SP: }&\, \mathop {\min }\limits_{\boldsymbol{b}[t]} \xi_L \left( {{\boldsymbol{b}}\left[ t \right]} \right)\label{eq:long-terma} \\
        &\st \quad \eqref{eq:Mainb}, \eqref{eq:Maind}, \eqref{eq:Maine}, \eqref{lemma1}, \eqref{lemma2}, \eqref{lemma3} \label{eq:long-termb}\quad\\
        &\qquad\, {\boldsymbol{b}}\left[ t \right] \in  \boldsymbol{\mathscr{B}}'[t]\label{eq:long-termc} 
    \end{IEEEeqnarray}
\end{subequations}
where $\boldsymbol{\mathscr{B}}'[t]\buildrel \Delta \over =\big\{\boldsymbol{b}[t]\Bigl|\, \eqref{setb_a}, \eqref{setb_b}, \eqref{setb_e} \big\} $. Here, constraints \eqref{setb_c} and \eqref{setb_d} are replaced by the tighter structural condition in \eqref{lemma2}, which reduces redundancy and simplifies the feasible region.

\textit{2) Short-term subproblem (S-SP):} Given the optimal solution $\boldsymbol{b}^*$ obtained from \eqref{eq:long-term}, the short-term subproblem optimizes task offloading and UE transmit power, and is formulated as
\begin{subequations} \label{eq:short-term}
	\begin{IEEEeqnarray}{cl}
		 \textbf{S-SP:}&\mathop {\min }\limits_{{\boldsymbol{\beta }[t_s]},{\boldsymbol{p}[t_s]}} \xi_S \left( {{\boldsymbol{\beta }}\left[ {{t_s}} \right],{{\boldsymbol{p}}}\left[ {{t_s}} \right]} \right)\label{eq:short-terma} \\
		& \st \quad  \eqref{eq:Mainb}-\eqref{eq:Maind}, \eqref{setphi}, \eqref{setp}  \label{eq:short-termb}.
	\end{IEEEeqnarray}
\end{subequations}

\section{Two-Timescale Latent-Action DRL}\label{sec_Solutions} 
The resulting subproblems in~\eqref{eq:long-term} and \eqref{eq:short-term} remain constrained and highly non-convex, making them difficult to solve using conventional methods, especially in large-scale networks. To address this challenge, we propose 2T-MDRL-LA, a two-timescale multi-layer DRL framework with a latent action space. For the long-term subproblem~\eqref{eq:long-term}, the large discrete action space is handled by a latent-action PPO (LA-PPO) agent, which reduces complexity while ensuring feasibility and improving convergence. For the short-term subproblem~\eqref{eq:short-term}, which involves continuous decision variables such as task offloading ratios and UE transmit power, a PPO-based agent is employed to adaptively optimize these variables at each time-slot $t_s$.


\subsection{Proximal Policy Optimization (PPO)} \label{PPO}
PPO is an on-policy actor–critic algorithm that ensures stable and efficient policy updates by constraining the deviation between successive policies. This makes PPO well-suited for the proposed two-timescale framework, where stable learning is required under dynamic environments and hybrid decision structures involving both discrete and continuous variables.
PPO optimizes a clipped surrogate objective to prevent overly large policy updates. The policy loss at time-slot $t$ is given by
{\small\begin{IEEEeqnarray}{cl}\label{eq:clip_objective}
\mathcal{L}_{\pi}(\theta) = \mathbb{E}_t \!\left[ \min\!\left( r_t(\theta)\hat{A}_t,\; \operatorname{clip}\!\left(r_t(\theta),\,1-\epsilon,\,1+\epsilon\right)\hat{A}_t\right) \right]\quad\,\,
\end{IEEEeqnarray}}where $r_t(\theta)=\exp\!\left(\log\pi_\theta(a_t\mid s_t)-\log\pi_{\theta_{\mathrm{old}}}(a_t\mid s_t)\right)$ denotes the probability ratio between the current and previous policies for selecting action $a_t$ in state $s_t$. The clipping parameter $\epsilon$ controls the update step size and prevents instability during training. The advantage estimate $\hat{A}t$ is computed using Generalized Advantage Estimation (GAE), \textit{i.e.} $\hat{A}_t = \sum_{l=0}^{N_{\mathrm{rol}}-t-1} (\gamma\lambda)^l \,\delta_{t+l}$, where $\gamma$ is the discount factor, $\lambda$ governs the bias–variance trade-off, and $N_{\mathrm{rol}}$ denotes the rollout length. The temporal-difference (TD) error is defined as $\delta_t = r_t + \gamma V_\phi(s_{t+1}) - V_\phi(s_t),$ where $r_t$ denotes the immediate reward at time-slot $t$, and $V_\phi(\cdot)$ denotes the state-value function estimated by the critic.

PPO adopts an actor–critic architecture with a shared feature extractor, enabling joint policy learning and value function approximation to improve sample efficiency and reduce gradient variance. The overall training objective is given by
\begin{IEEEeqnarray}{cl}\label{eq:total_loss_PPO}
\mathcal{L}_{\mathrm{total}}(\theta,\phi) = \mathcal{L}_{\pi}(\theta) - \lambda_V \mathcal{L}_V(\phi) + \lambda_H \mathcal{H}\!\left[\pi_\theta(\cdot  \mid  s_t)\right]\qquad
\end{IEEEeqnarray}
where $\mathcal{L}_V(\phi)$ denotes the critic loss, defined as the mean squared error between $V_\phi(s_t)$ and the target return. The entropy term $\mathcal{H}[\cdot]$ encourages exploration, while $\lambda_V$ and $\lambda_H$ are weighting coefficients that balance value function learning and exploration.

\subsection{Long-term Action Space Design} \label{Longterm_action_design}
Each action in the long-term action space $\mathscr{A}_L$ specifies the system configuration at time-slot $t$, including user association, service placement, and ES–ES and ES–CS cooperation. The long-term action vector is defined as
\begin{equation}
    \boldsymbol{a}^L[t] = \boldsymbol{b}[t] \in \mathscr{A}_L
\end{equation}
which determines the service placement and coordination policy. The action space is binary-valued,
$\mathscr{A}_L \subseteq \{0,1\}^{D_{\mathrm{full}}}$ with $ 
    D_{\mathtt{full}} = MK + NK + MK^2 + MK$,
and cardinality $\big| \mathscr{A}_L \big| = 2^{D_{\mathtt{full}}}$,  which becomes intractable in large-scale systems.

\begin{figure}[t]
	\centering
	\includegraphics[width=0.9\columnwidth,trim={0cm 0.0cm 0cm 0.0cm}]{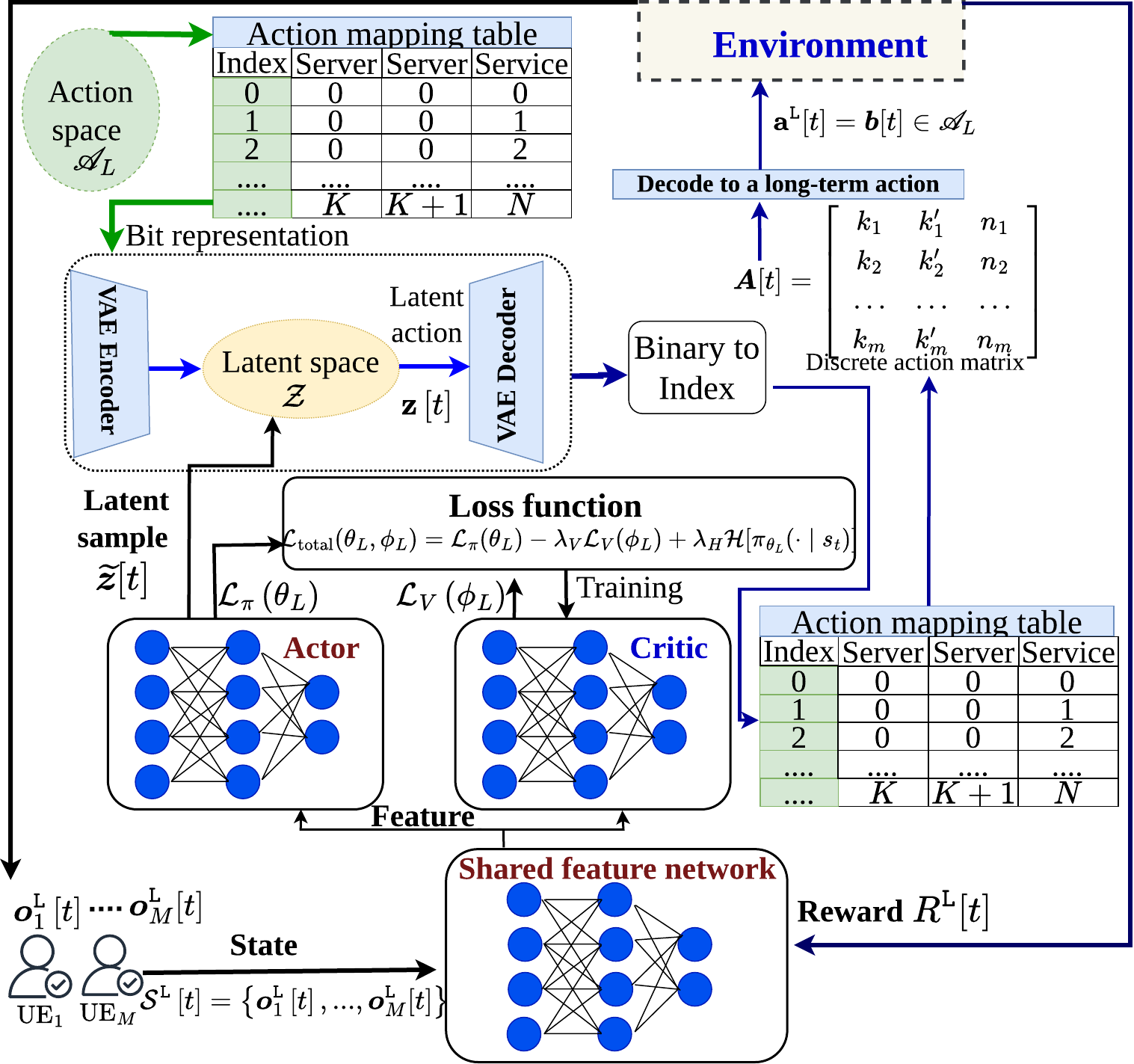}
	\caption{Long-term decision phase with LA-PPO.}
	\label{fig:2T_MADRL_model}
\end{figure}

To alleviate this issue, as illustrated in Fig.~\ref{fig:2T_MADRL_model}, the full action space is first transformed into an action mapping table $\mathcal{V}_{3 \times K(K+1)N}$, where each row corresponds to a candidate action $\boldsymbol{a}[t]$ defined by a triplet of the associated server, executing server, and deployed service. The agent selects actions by indexing this table based on its observation. However, the complexity still scales as $\mathcal{O}(NK^2)$.

To further reduce dimensionality, we introduce a latent action representation based on a VAE~\cite{chen2024resource}. 
The encoder $\delta(\cdot)$ and decoder $\psi(\cdot)$ are defined as
\begin{align}\label{eq:encoder_decoder_VAE}
    \delta: \mathcal{V} \mapsto \mathcal{Z}, 
    \qquad
    \psi: \mathcal{Z} \mapsto \mathcal{V}
\end{align}
where $\mathcal{Z}$ is a low-dimensional latent space with $|\mathcal{Z}| \ll |\mathcal{V}|$. Each binary vector $\boldsymbol{v}[t] \in \{0,1\}^{|\boldsymbol{v}|}$ representing the index of an action $\boldsymbol{a}[t]$ is encoded into a latent vector $\boldsymbol{z}[t] = \delta(\boldsymbol{v}[t])$. 
\begin{itemize}
\item For PPO-based continuous-action algorithms, the policy directly outputs a continuous latent vector $\boldsymbol{z}[t] \in \mathcal{Z}$.

\item For discrete-action algorithms (\textit{e.g.} DDQN), exploration is introduced by perturbing the latent vector as
\begin{align}
    \boldsymbol{z}[t] 
    = \{z'_i\}_{i=1}^{|\boldsymbol{z}[t]|}
    + \{\eta_i\}_{i=1}^{|\boldsymbol{z}[t]|}
\end{align}
where $z'_i\in\{0,1\}$ represents the discrete latent components and $\eta_i\sim\mathcal{CN}(0,1)$ denotes additive exploration noise.
\end{itemize}
The latent vector $\boldsymbol{z}[t]$ is then decoded via $\psi(\cdot)$ to obtain a binary representation $\boldsymbol{v}[t] = \psi(\boldsymbol{z}[t])$. This binary vector is then mapped to a scalar action index as
\begin{align}\label{eq:binary_to_scalar_index}
u = \sum_{i=1}^{|\boldsymbol{v}[t]|} \boldsymbol{v}_i2^{i-1}
\end{align}
where $u \in \{0,1,\ldots,2^{|\boldsymbol{v}[t]|}$ $-1\}$ identifies the corresponding entry in the action mapping table, from which $\boldsymbol{a}[t]$ is retrieved.

\subsection{Proposed 2T-MDRL-LA Framework }\label{sec_MDRL_Solutions}

\begin{algorithm}[t]
\small
    \setstretch{0.95}  
    \begin{algorithmic}[1]
        \protect\caption{The Overall Algorithm for Solving Problem~\eqref{eq:Main}} \label{alg_1}
        \global\long\def\algorithmicrequire{\textbf{Initialization:}}    
        \REQUIRE
         Set $T, S$ and initialize the environment for problem \eqref{eq:Main}.
        \STATE Set \textit{long-term flag} = $\mathtt{TRUE}$;
        \global\long\def\algorithmicrequire{\textbf{Main Loop:}}
        \REQUIRE 
        \FOR{each frame $t = \{1,2,\cdots,T\}$}
            \IF{\textit{long-term flag} is $\mathtt{TRUE}$}
                \STATE Solve \textbf{L-SP} \eqref{eq:long-term} via the long-term agent to obtain  ${{\boldsymbol{b}}^*}[t]$;
                \STATE Set \textit{long-term flag} = $\mathtt{FALSE}$;
            \ENDIF
            \FOR{each time-slot ${t_s} = \{{t_1},{t_2},\cdots,{t_S}\}$}
                \STATE Given $\boldsymbol{b}$, solve \textbf{S-SP} \eqref{eq:short-term} via the short-term agent to obtain (${{\boldsymbol{\beta }}^*[t_s]},$ ${{\bf{p}}^*[t_s]}$);
                \IF{a new requested service is unavailable at ESs or latency constraints are violated} 
                    \STATE Set \textit{long-term flag} = $\mathtt{TRUE}$;
                    \STATE \textbf{break}
                \ENDIF
            \ENDFOR
        \ENDFOR
    \end{algorithmic}
\end{algorithm}

\begin{figure}[t]
	\centering
	\includegraphics[width=0.95\columnwidth,trim={0cm 0.0cm 0cm 0.0cm}]{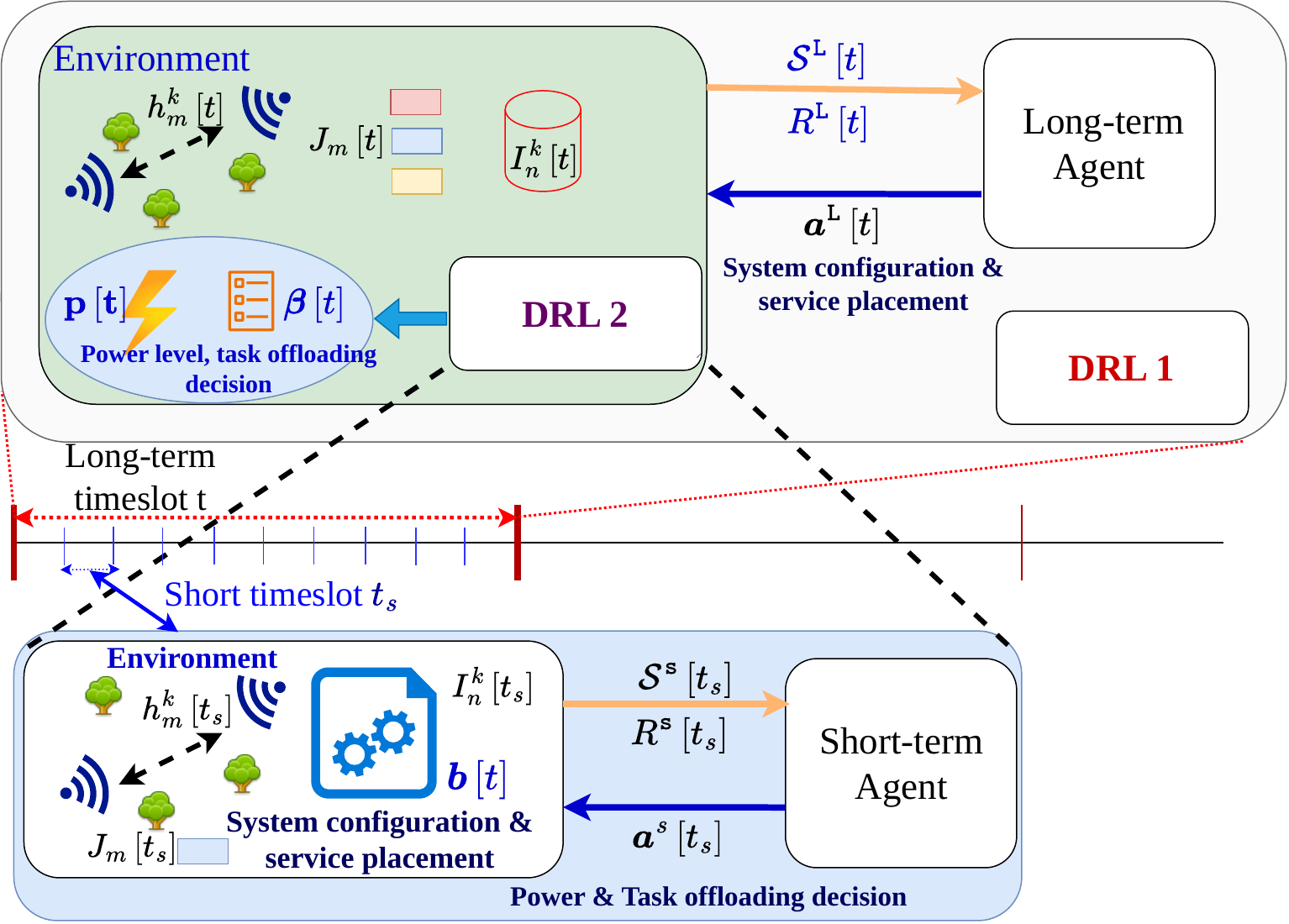}
	\caption{The proposed 2T-MDRL-LA framework.}
	\label{fig:2T_MADRL_Alg}
\end{figure}

Algorithm~\ref{alg_1} outlines the key steps for solving problem~\eqref{eq:Main} via a two-timescale strategy. In Step 4, the long-term subproblem~\eqref{eq:long-term} is solved to determine the system configuration, which remains fixed during the subsequent short-term slots. The short-term agent then optimizes task offloading and user transmit power by solving~\eqref{eq:short-term} in Step 8 at each time slot.
If new service requests arise or latency requirements are not satisfied, the long-term configuration is updated by resetting the long-term flag. This interaction enables adaptive optimization under dynamic network conditions while reducing computational complexity. 

The detailed solution procedures for the long-term and short-term agents are presented in the following subsections.

\textit{1) Long-term Agent for Solving \textbf{L-SP} \eqref{eq:long-term}}: As shown in Fig.~\ref{fig:2T_MADRL_model}, the long-term action determines the system configuration, including user association, service placement, and computational delegation. The corresponding binary vector $\boldsymbol{b}^*[t]$ is obtained via the proposed LA-PPO agent.\vspace{3pt}

\noindent\textit{\textbf{Observation:}} The long-term decision depends on service demand, channel conditions, user transmission behavior, and system congestion. At the beginning of each time-slot $t$, each UE $m$ reports task information $J_{m,n}[t]$, transmit power $p_m[t]$, and offloading decision $\beta_m[t]$, together with channel gain $\boldsymbol{g}_m[t]$ and queue length vector $\boldsymbol{i}[t]$ at ESs and CS. The observation for UE $m$ is
$\boldsymbol{o}_m^L = \{ J_{m,n}[t], \boldsymbol{g}_m[t], p_m[t], \beta_m[t], \boldsymbol{i}[t] \}$, and the global observation is
\begin{align}
    \label{eq:all_users_states}
    \mathcal{S}^L[t] = \{\boldsymbol{o}_1^L[t], \boldsymbol{o}_2^L[t],\ldots, \boldsymbol{o}_M^L[t]\}.
\end{align}

\noindent\textit{\textbf{Action:}} Given the observation $\mathcal{S}^L[t]$, the long-term action $\boldsymbol{a}^L[t]$ is generated by combining the PPO-based agent, the VAE, and the action mapping table. As shown in Fig.~\ref{fig:2T_MADRL_model}, the observation is first processed by a shared feature network to extract a latent representation, improving sample efficiency and reducing redundant feature learning. The actor head then parameterizes the policy, while the critic head estimates the state value function to stabilize training.

The model outputs a sampled latent vector, which is passed through an element-wise sigmoid function to obtain the latent action vector $\boldsymbol{\widehat{z}}[t]$. Specifically, $\boldsymbol{\widehat z}[t] = \{\boldsymbol{z}_m[t]\,\Bigl|\boldsymbol{z}_m[t] \in \mathcal{Z},\, m\in\mathcal{M}\} $, where each $\boldsymbol{z}_m[t]$ corresponds to the observation of UE $m$. This latent vector is then decoded into a binary vector via $\psi(\cdot)$ in~\eqref{eq:encoder_decoder_VAE}, and subsequently mapped to an action $\boldsymbol{a}[t] \in \mathcal{V}$ using~\eqref{eq:binary_to_scalar_index}.

After applying the mapping mechanism, the resulting action matrix $\boldsymbol{A}[t]$ for all UEs is given by
\begin{align}
\label{eq:all_actions}
\boldsymbol{A}[t] = \left[
\begin{array}{ccc}
k_1 & k'_1 & n_1 \\
k_2 & k'_2 & n_2 \\
\vdots & \vdots & \vdots \\
k_M & k'_M & n_M
\end{array}
\right]\in \mathbb{R}^{M \times 3}
\end{align}
where each row specifies the associated server $k \in \mathcal{K}$, the execution server $k' \in \mathcal{K}$, and the deployed service $n \in \mathcal{N}$. Finally, the long-term action $\boldsymbol{a}^L[t] \in \mathscr{A}_L$ is obtained from $\boldsymbol{A}[t]$.\vspace{3pt}

\begin{algorithm}[t]
\footnotesize
\LinesNumbered
\caption{LA-PPO-Based Long-Term Training}
\label{alg:PPO-longterm}

\SetKwInOut{Initialization}{Initialization}

\Initialization{Set hyperparameters $\gamma, \lambda, \epsilon_{\text{clip}}, N_{\text{rol}}, K_{\text{epoch}}, B$; initialize actor $\theta_L$, critic $\phi_L$, and old policy $\theta_L^{\text{old}} \leftarrow \theta_L$; pre-train VAE and construct action mapping table.
}

\For{$\text{episode} = 1,\dots,N_{\text{episodes}}$}{
    \textbf{Reset:} Collect all users' states $\mathcal{S}^\mathtt{L}[0]$ according to~\eqref{eq:all_users_states}, $t \leftarrow 0$, $\text{done} \leftarrow \text{False}$\;

    \While{\textbf{not} $\text{done (each step t)} $}{
   
        Decode latent action $\boldsymbol{z}[t]$ to obtain $\boldsymbol{A}[t]$ via~\eqref{eq:encoder_decoder_VAE}-\eqref{eq:binary_to_scalar_index} and mapping;

       Convert $\boldsymbol{A}[t]$ to system configuration $\boldsymbol{a}^L[t]$ and execute;

       Observe next state $\mathcal{S}^L[t+1]$, global reward $R^L[t]$, individual rewards $\{r_m^L[t]\}$, and $\textit{done}$;

       Store transition $(\boldsymbol{o}_m^L[t], \boldsymbol{z}_m[t], \log \pi_{\theta_L}(\tilde{\boldsymbol{z}}_m[t]|\boldsymbol{o}_m^L[t]),$ $ r_m^L[t], V_{\phi_L}(\boldsymbol{o}_m^L[t]), \textit{done})$ in buffer $B$;\\

       Set  $\mathcal{S}^\mathtt{L}[t] \leftarrow \mathcal{S}^\mathtt{L}[t+1], \, t \leftarrow t + 1$\;

    }
     \If{$| B|  \ge  N_{\text{rol}}$ }{
            Compute returns and advantages $\hat{A}_t$ using GAE\;

            \For{$k = 1,\dots,K_{\text{epoch}}$}{
                Sample minibatches of size $B$\;
                Update $\theta_\mathtt{L}$ using \eqref{eq:clip_objective} ;
                Update $\phi_\mathtt{L}$ using the value loss defined in~\eqref{eq:total_loss_PPO}\;
            }
            $\theta_{\mathtt{L}_{\text{old}}} \leftarrow \theta_\mathtt{L}$\;
            Clear buffer $B$;
        }
}
\end{algorithm}
\noindent\textit{\textbf{Reward:}} At the beginning of each time-slot $t$, the long-term agent applies the action $\boldsymbol{a}^L[t]$ to the environment, thereby determining the system configuration, including user association, service placement, and computational delegation, under the given user offloading decisions and transmit power levels. During time-slot $t$, the system evolves over $S$ short-term time slots.

At the end of the long-term time slot, the e2e latency experienced by each UE across the $S$ short-term slots is collected and averaged. This aggregated latency information is then fed back to the agent to evaluate the quality of the selected long-term action. Accordingly, the reward is defined as the negative average e2e latency of all users, given by
 \begin{equation}\label{eq:reward_longterm}
     R^\mathtt{L}[t] = -\frac{1}{M} \sum_{m\in \mathcal{M}}\overline{\tau}_{m}^{\mathtt{e2e}}\left[ {{t}} \right] 
 \end{equation}
where $\overline{\tau}_{m}^{\mathtt{e2e}}\left[ {{t}} \right] = \sum_{s \in S}{\tau_m^{\mathtt{e2e}}[t_s]}/S$ denotes the average e2e latency of UE $m$ over the $S$ short-term slots within the long-term time-slot $t$.

To facilitate training of the LA-PPO-based long-term agent, each UE $m$ collects transition tuples of the form
$\left(\boldsymbol{o}_m^\mathtt{L}[t],\; \boldsymbol{z}_m[t],\; \log \pi_\theta( \boldsymbol{\widetilde z}_m[t] \mid \boldsymbol{o}_m^\mathtt{L}[t]),\; r_m^\mathtt{L}[t],\; V_\phi(\boldsymbol{o}_m^\mathtt{L}[t]) \right)$
which are stored in the learning buffer. Here, $r_m^L[t]$ represents the reward associated with UE $m$ (derived from the average e2e latency), and $\log \pi_\theta( \boldsymbol{\widetilde z}_m[t] \mid \boldsymbol{o}_m^\mathtt{L}[t])$ denotes the log-probability of selecting the sampled latent action $\boldsymbol{\widetilde{z}}_m[t]$ under the previous policy. The LA-PPO-based long-term agent is then trained according to the procedure described in Algorithm~\ref{alg:PPO-longterm}.\vspace{3pt}

\begin{algorithm}[t]
\footnotesize
\LinesNumbered
\caption{PPO-based Short-term Training}
\label{alg:shortterm-training}
\SetKwInOut{Initialization}{Initialization}
\Initialization{Set $\gamma, \lambda, \epsilon_{\text{clip}}, N_{\text{rol}}, B$; load trained long-term model; initialize $\theta_s, \phi_s$, and $\theta_s^{\text{old}} \leftarrow \theta_s$.}

\For{$\text{episode} = 1,\dots,N_{\text{episodes}}$}{
    Reset environment\;
    Obtain long-term environment state $\mathcal{S}^\mathtt{L}[t]$\;
    Generate system configuration $\boldsymbol{a}^\mathtt{L}[t]$ using long-term policy $\pi_{\theta_\mathtt{L}}$ and $\mathcal{S}^\mathtt{L}[t]$\;
    {Execute $\boldsymbol{a}^L[t]$}\;
    $\text{done} \leftarrow \text{False}$, $t_s \leftarrow 0$\;

    \While{\textbf{not} $\text{done (each step $t_s$)}$}{
        Execute short-term actions $\{\boldsymbol{a}_m^{\mathtt{s}}[t_s]\}_{m=1}^M$\;
        Observe global reward $R^s[t_s]$, individual rewards $\{r_m^s[t_s]\}_{m=1}^{M}$, and \textit{done}\;
        Store  $\big(\boldsymbol{o}_m^\mathtt{s}[t_s],\; \boldsymbol{a}_m^\mathtt{s}[t_s],\; \log \pi_{\theta_\mathtt{s}}(\boldsymbol{a}_m^\mathtt{s}[t_s] \mid \boldsymbol{o}_m^\mathtt{s}[t_s]),\;$ $r_m^\mathtt{s}[t_s],\;  V_{\phi_\mathtt{s}}(\boldsymbol{o}_m^\mathtt{s}[t_s]),\; \textit{done}\big)$ in buffer $B$ for all UEs\; 
        Set $t_s \leftarrow t_s + 1$\;

    }
    \tcp{Policy update}
        Update short-term policy $\pi_{\theta_s}$ following steps (11)-(19) in Algorithm ~\ref{alg:PPO-longterm}\;
        Set $t \leftarrow t + 1$\;
}
\end{algorithm}

 \textit{2) Short-term Agent for Solving \textbf{S-SP}~\eqref{eq:short-term}}: To solve subproblem~\eqref{eq:short-term}, which involves mixed binary (offloading decisions $\boldsymbol{\beta}[t_s]$) and continuous (transmit power $\boldsymbol{p}[t_s]$) variables, we employ a PPO-based short-term agent to minimize the average e2e latency. The state, action, and reward are defined as follows.

\noindent\textit{\textbf{ State:}}
At the beginning of each short-term time-slot $t_s$, the agent observes
\begin{align}\label{eq:short_term_state}
    \mathcal{S}^{\mathtt{s}}[t_s] = \{\boldsymbol{o}_m^{\mathtt{s}}[t_s]\}_{m \in \mathcal{M}}
\end{align}
where $\boldsymbol{o}_m^{\mathtt{s}}[t_s]$ includes the effective channel gain $\boldsymbol{g}_m[t_s]$, task information $J_{m,n}[t_s]$, system configuration $\boldsymbol{b}[t]$, and queue lengths $\boldsymbol{i}[t_s]$ at ESs and CS.

\noindent\textit{\textbf{ Action:}} For each UE $m$, the agent jointly determines the offloading decision $\beta_m[t_s]$ and transmit power $p_m[t_s]$. To handle the mixed variables, the transmit power is discretized into ten levels $\hat{p}_m[t_s] \in \{1,\ldots,10\}$, such that
\begin{align}
    p_m[t_s] = \frac{\hat{p}_m[t_s]}{10} P_m^{\max}.
\end{align}
The action vector is defined as
\begin{align}\label{eq:short_term_action}
    \boldsymbol{a}^{\mathtt{s}}[t_s] \buildrel \Delta \over = \{\boldsymbol{a}_m^s[t_s]\}_{m \in \mathcal{M}}
\end{align}
where $\boldsymbol{a}_m^s[t_s] = \{\beta_m[t_s], \hat{p}_m[t_s]\}_{m \in \mathcal{M}}$ is obtained using the policy $\pi_{\theta_s}$.

\noindent\textit{\textbf{ Reward:}}
The global reward equals the negative average e2e latency of all UEs, such as
\begin{align}
    R^{\mathtt{s}}[t_s] = -\frac{1}{M} \sum_{m \in \mathcal{M}}\tau_m^{\mathtt{e2e}}[t_s]
\end{align}
where $-\tau_m^{\mathtt{e2e}}[t_s] \equiv r_m^s[t_s]$ is the reward for an action $\boldsymbol{a}_m^\mathtt{s}[t_s]$.
The training procedure is given in Algorithm~\ref{alg:shortterm-training}.

\section{Numerical Results}\label{sec_NumericalResults}
 \subsection{Simulation Setting}\label{sec_simulation_setup}
\textbf{\textit{Network settings:}} We consider an HECC-aided IoT network where APs, ESs, and UEs are deployed over a $100 \times 100$~m area~\cite{huynh_latency}. The CS is located at a distance of $d_{k,c} = 10$~km from each ES $k$. The ESs are positioned at $(50,0)$, $(0,50)$, $(-50,0)$, and $(0,-50)$ for $K=4$, and at $(50,0)$ and $(-50,0)$ for $K=2$. The wireless channel is characterized by large-scale fading, given by
$ g_m^k[t] = 10^{\frac{\mathrm{PL}(d_{mk}[t])}{10}},$ where the path loss (in dB) is defined as $\mathrm{PL}(d_{mk}) = -35.3 - 37.6 \log_{10}(d_{mk}[t]),$ with $d_{mk}[t]$ denoting the distance between UE $m$ and ES $k$~\cite{H.D.Tuan}. Moreover, the single-sided noise power spectral density is set to $-174$~dBm/Hz~\cite{Dinh:JSAC:Dec2017}.

We set the number of long-term frames to $T = 50$. UEs request services from ESs or the central cloud every five frames. The service request price for data transfer from ESs to the CS is set to $\rho^\mathtt{cs} = \$0.01$, following AWS intra-region data transfer pricing (US East–New York City)~\cite{aws_ec2_pricing}. This pricing model can be extended to other ES/CS providers. The remaining simulation parameters are summarized in Table~\ref{tab:Simulationparameter1}.

\textbf{\textit{PPO and BnB setup:}} The proposed 2T-MDRL-LA framework is implemented and trained in PyTorch. All experiments are conducted on a desktop computer running Windows~11 with a 3.5~GHz Intel Core~i9 processor and 64~GB RAM. The BARON solver, integrated with the YALMIP toolbox, is used to solve the BnB optimization problems.

To further evaluate the proposed method under service demand surges and load imbalance, we consider heterogeneous ES computational capacities of $[30, 40, 50, 60]$~GHz. The following scenarios are examined:
\begin{itemize}
    \item \textbf{Scenario~1:} UE locations follow a normal distribution, and service requests are updated every 10 long-term frames.
    \item \textbf{Scenario~2:} UE-ES distances vary after one long-term frame, where $20\%$ of UEs are located near ES~3 and ES~4, and $80\%$ are concentrated near ES~1 and ES~2.
\end{itemize}

 \begin{table}[t]
\centering  
\captionof{table}{Simulation Parameters}
\label{tab:Simulationparameter1}
\scalebox{0.75}{
    \begin{tabular}{l|l}
        \hline
        Parameter & Value \\
        \hline\hline
        System bandwidth, $W$ & 10 MHz \cite{huynh_latency}\\
        Maximum number of services, $N$  & 6 \cite{huynh_service_placement}\\
        Maximum installed services at ES, $N_k^{\max},\, \forall k$  & 4 \cite{huynh_service_placement}\\
        Processing rate of UE, $f_m$\, $\forall$ m & 1 GHz \cite{huynh_service_placement}\\
         Maximum computing capacity of ESs, $f_{\max }^k$ & $\left[ {30,60} \right]$ GHz\cite{Zhou_cost}\\
        Maximum computing capacity of CS, ${f_{\max}^{\cs}}$ & 100 GHz\\
        Backhaul $\&$ fronthaul capacity,
        $R_k^{k'}$ $\&$ $R_k^c$,\,$\forall k$ & 5 $\&$ 1 Gbps \cite{huynh_latency} \\
        Size of task, ${D_{m}},\,\forall m$ &  [0.1, 0.5] MBits \\
        Maximum delay requirement, $\tau_m^{\max },\,\forall m$ & 0.2 s \cite{huynh_service_placement}\\
        Required computation resource, $C_{m},\forall m$  & $\left[ {100,200} \right] \, \text{Mcycles} $ \cite{huynh_latency} \\
        Maximum UE's transmitted power, ${P_m^{\max}},\forall m$ & 23 dBm \cite{H.D.Tuan} \\
        Minimum data rate requirement, ${R_{\min }}$ & 1 Mbps\cite{huynh_latency} \\
        Price for installing service, ${\rho ^\mathtt{i}}$ & 0.1 $\$/$service \cite{Jiaming_2024} \\
        Price for uninstalling service, ${\rho ^\mathtt{u}}$ & 0.05 $\$$/service \cite{Jiaming_2024}  \\
        Price for operating service, ${\rho ^\mathtt{o}}$ & 0.1 $\$$/service \cite{Jiaming_2024}  \\
        Maximum average monetary cost, ${c^\mathtt{{\max }}}$ & ${\$}$4  \\
        UE's Maximum energy consumption, $e_m^{\max }$ & 1 Joule \cite{huynh_latency}\\
        Effective capacitance coefficient, ${\mu _m}$ & ${10^{ - 27}}\,\frac{Watt.{s^3}}{cycl{e^3}}$ \cite{Liu_resource_frame} \\
        
        Discount factor and GAE parameter, $\gamma \,\& \, \lambda $  & 0.99 $\&$ 0.95\\
        PPO clipping parameter, $\epsilon_{\text{clip}}$ & 0.2 \\
        Long-term and short-term PPO learning rates & $3\times10^{-3} \& 3\times10^{-4}$  \\
        Adam optimizer & (weight decay $10^{-5}$) \\
        Rollout length $\&$ minibatch size, $N_{\text{rol}}$ $\&$ B & $4096$ $\&$ 512\\

        \hline		  				
    \end{tabular}
}
\end{table}
\subsection{Benchmark Schemes}\label{sec_benchmark_schemes}
To evaluate the performance of the proposed method, we compare it with the following benchmark schemes:
\begin{itemize}
    \item ``LA-DDQN-DDQN":  L-SP is solved using DDQN with latent action (LA-DDQN), while  S-SP is solved using DDQN. 
    \item ``LA-DDQN-PPO": L-SP is addressed by LA-DDQN, whereas S-SP is solved using PPO.
 
    \item ``LA-PPO-DDQN": L-SP is solved using LA-PPO, while S-SP is handled by DDQN.

    \item  ``PPO (w/o LA)" and ``DDQN (w/o LA):  PPO and DDQN are applied to L-SP without using the LA space or mapping table. These serve as baselines to evaluate the impact of the latent action design.

    \item ``Random User Association (RUA)": User association is randomized, while the remaining decisions (service placement, power allocation, and task offloading) are optimized using PPO.

    \item ``w/o CDO": Computational delegation is disabled; tasks are executed only at the associated ES or locally at UEs.
    
    \item ``Without Service Placement Optimization (w/o SPO)": Service placement at ESs is fixed.
    
    \item ``Random Processing Task (Rand PT)": Task processing decisions in both ESs and CS are randomly assigned. 
\end{itemize}

\textit{Performance Metrics:} The performance is evaluated using four metrics: $(i)$ average global reward, $(ii)$ average latency per user, $(iii)$ task offloading ratio, and $(iv)$ resource utilization, defined as the proportion of CPU cycles effectively used for task processing. Low task loads at edge servers may lead to inefficient resource utilization. Following~\cite{Yang_2023}, the average resource utilization is defined as
\begin{equation}
\overline{\eta} = \frac{1}{TS \big(\sum_{k \in \mathcal{K}} f_{\max}^k + f_{\max}^{\mathtt{cs}}\big)} 
\sum_{t \in \mathcal{T}} \sum_{t_s \in \mathcal{S}} \widetilde{C}[t_s]
\end{equation}
where $\widetilde{C}[t_s]$ denotes the total number of CPU cycles processed during time-slot $t_s$.

\subsection{Numerical Results and Discussions}\label{sec_numerical_results_dicussions}
\begin{figure}[t]
	\centering
	\includegraphics[width=0.95\columnwidth,trim={0cm 0.0cm 0cm 0.0cm}]{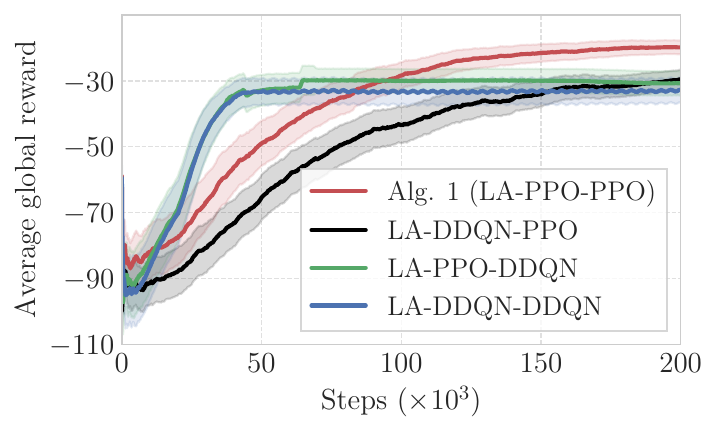}
	\caption{\small Convergence behavior of Alg.~\ref{alg_1} compared with other algorithms with $M = 30$.}
	\label{fig:overall_convergence_comparision}
\end{figure}

\begin{figure}[t]
    \centering
    \vspace{-2mm}
    \subfigure[Convergence of Alg.~\ref{alg:PPO-longterm} (LA-PPO) under different values of $\alpha$]{
        \includegraphics[width = 3in]{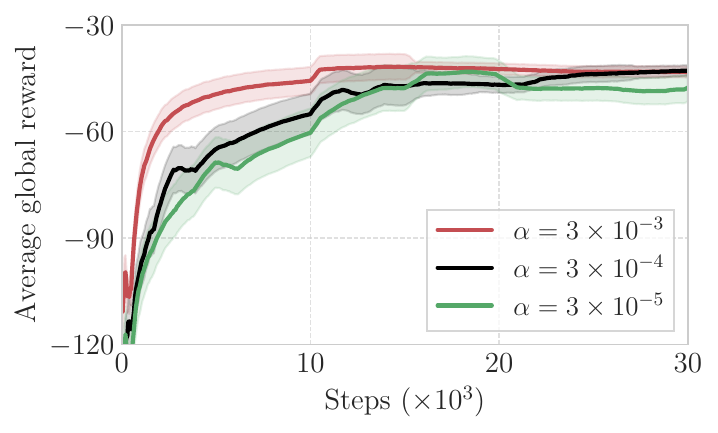}
        \label{fig:Agl2_convergence_learning_rate}
    }
    \subfigure[Convergence of Alg. ~\ref{alg:PPO-longterm} with different action space designs]{
        \includegraphics[width =3in]{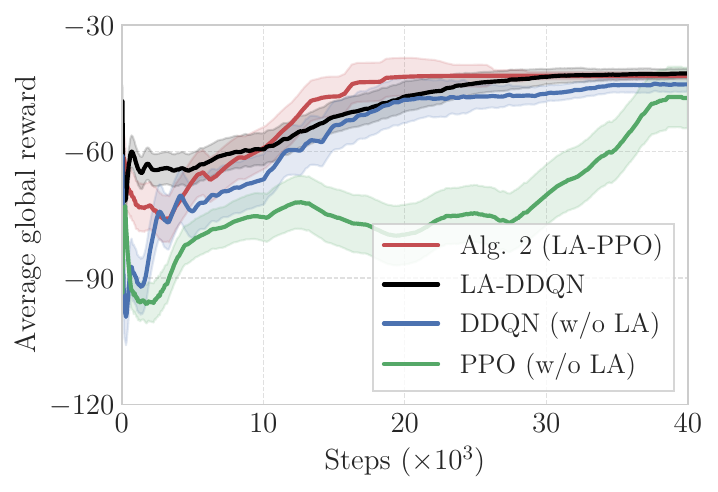}
        \label{fig:Agl2_convergence_latent_action}
    }
      \vspace{-2mm}
    \caption{\small Convergence analysis of Algorithm \ref{alg:PPO-longterm} with $M = 100$ . }
    \label{fig:Agl2_convergence}
\end{figure}

\textit{1) Convergence performance:} We first evaluate the convergence behavior of different DRL algorithms. As shown in Fig.~\ref{fig:overall_convergence_comparision}, the proposed Algorithm~\ref{alg_1} (LA-PPO-PPO) achieves the highest cumulative reward among all schemes. Although LA-PPO-DDQN and LA-DDQN-DDQN exhibit faster convergence, they converge to lower reward values due to suboptimal binary decision-making in both the long-term and short-term subproblems.

In Fig.~\ref{fig:Agl2_convergence}, we evaluate the effects of learning rate and action space design with $M = 100$. As shown in Fig.~\ref{fig:Agl2_convergence_learning_rate}, three learning rates $(\alpha = 3 \times 10^{-5},\, 3 \times 10^{-4},\, 3 \times 10^{-3})$ converge to similar high reward levels, while the largest learning rate achieves faster convergence. The impact of the LA design is illustrated in Fig.~\ref{fig:Agl2_convergence_latent_action}. Overall, Algorithm~\ref{alg:PPO-longterm} (LA-PPO) outperforms the compared approaches. The results show that the latent action space with mapping effectively reduces the action dimensionality when integrated with PPO (LA-PPO) and DDQN (LA-DDQN), leading to higher rewards and faster convergence at approximately $20 \times 10^{3}$ and $25 \times 10^{3}$ steps, respectively. In contrast, conventional PPO (w/o LA) and DDQN (w/o LA) achieve lower rewards and converge more slowly, reaching saturation only after about $40 \times 10^{3}$ and $35 \times 10^{3}$ steps, respectively. This demonstrates that Algorithm~\ref{alg:PPO-longterm} improves convergence speed by approximately $50\%$ compared to PPO (w/o LA).

\begin{figure}[t]
	\centering
	\includegraphics[width=0.95\columnwidth,trim={0cm 0.0cm 0cm 0.0cm}]{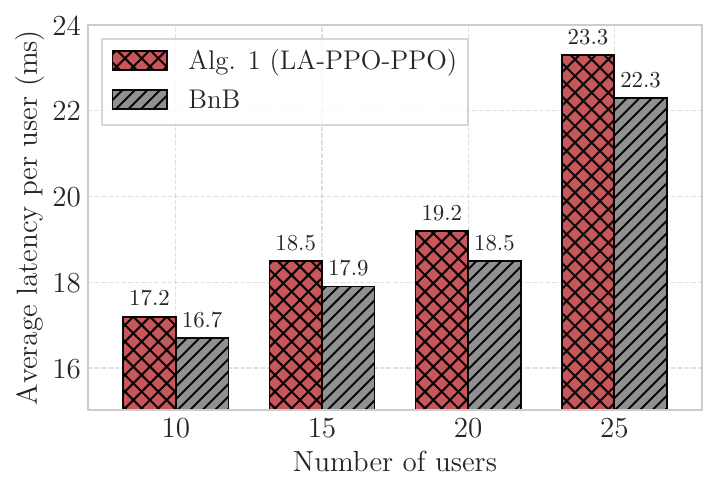}
	\caption{\small Comparison of average e2e latency per UE with BnB.}
	\label{fig:comparison_BnB}
\end{figure}

\textit{2) Performance comparison:} Fig.~\ref{fig:comparison_BnB} compares the proposed Algorithm~\ref{alg_1} with the optimal BnB method under the same input data, where BnB serves as the benchmark for global optimality. As shown, Algorithm~\ref{alg_1} yields a slightly higher average e2e latency per user, with an observed optimality gap of approximately $4\%$ over $M \in [10,25]$. This is because PPO does not explicitly explore the full combinatorial solution space; however, it learns effective policies that consistently produce high-quality solutions in practice. Consequently, the proposed method achieves near-optimal performance with significantly lower computational complexity than BnB.

\begin{figure}[t]
    \centering
    \vspace{-2mm}
    \subfigure[Comparison with baseline DRL schemes]{
        \includegraphics[width = 2.9in]{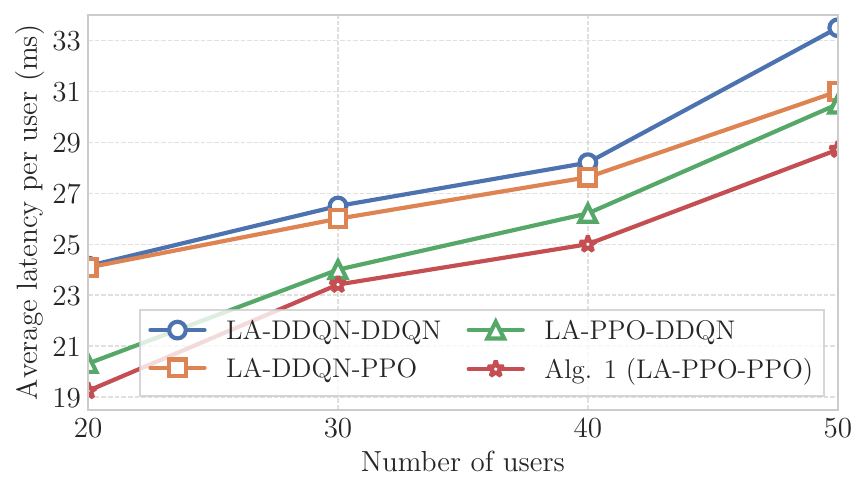}
        \label{fig:comparison_latency_baseline_ue}
    }
    \subfigure[Comparison with benchmark schemes]{
        \includegraphics[width =2.9in]{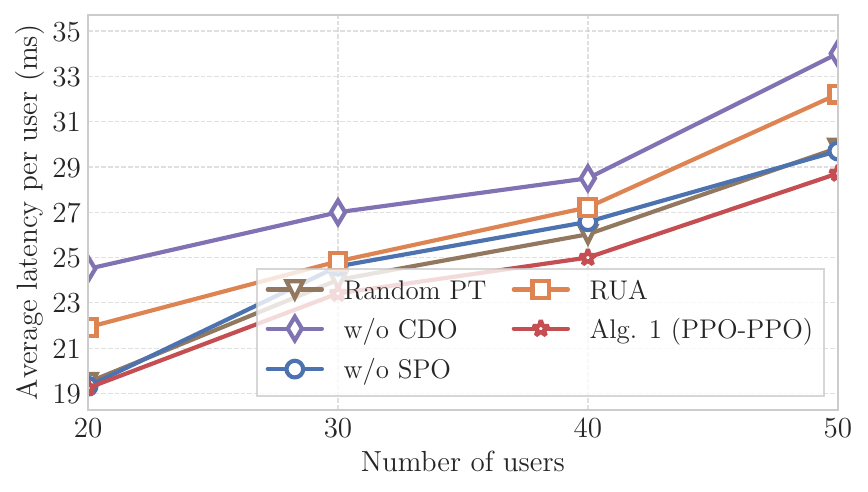}
        \label{fig:comparion_latency_benchmark_ue}
    }
      \vspace{-2mm}
    \caption{\small Average e2e latency per UE versus number of users.}
    \label{fig:comparison_latentcy_versus_UE}
\end{figure}

\begin{figure}[t]
	\centering
	\includegraphics[width=0.95\columnwidth,trim={0cm 0.0cm 0cm 0.0cm}]{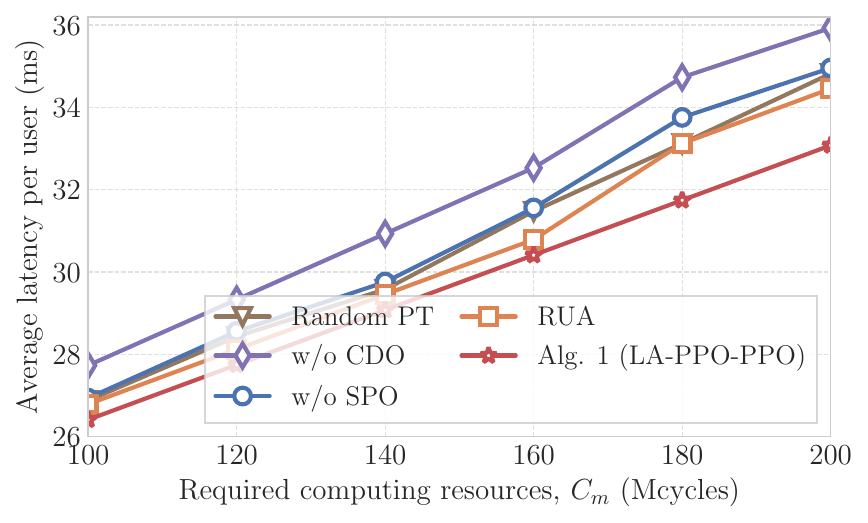}
	\caption{\small Average e2e latency per user versus required computational resources ($C_m$) under different benchmark schemes, with $M = 50$.}
	\label{fig:comparison_latency_verus_cycles}
\end{figure}

The impact of the number of UE on the average e2e latency across different schemes is shown in Fig.~\ref{fig:comparison_latentcy_versus_UE}. As $M$ increases, the average latency rises due to higher task arrival rates, increased congestion in ES–ES and ES–CS cooperation, and longer queueing delays at both ESs and  CS.
As shown in Fig.~\ref{fig:comparison_latency_baseline_ue} and Fig.~\ref{fig:comparion_latency_benchmark_ue}, Algorithm~\ref{alg_1} consistently outperforms both baseline and benchmark schemes, demonstrating the effectiveness of jointly optimizing service placement, user association, computational delegation, task offloading, and power allocation. Notably, Fig.~\ref{fig:comparion_latency_benchmark_ue} shows that Algorithm~\ref{alg_1} reduces the average latency by approximately $20.8\%$ compared to the w/o CDO scheme as $M$ increases from $20$ to $50$.
Fig.~\ref{fig:comparison_latency_verus_cycles} further shows that the average latency increases with the required computational resources $C_m$ due to higher processing delays, while Algorithm~\ref{alg_1}  consistently achieves the lowest latency among all compared schemes.

\begin{figure}[t]
    \centering
    \vspace{-2mm}
    \subfigure[Scenario 1]{
        \includegraphics[width = 3in]{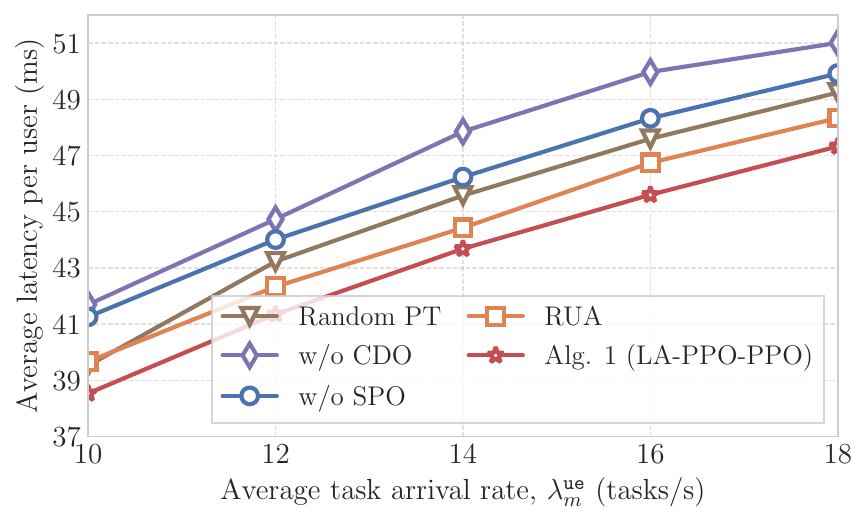}
        \label{fig:comparison_latency_task_arrival_rate_case1}
    }
    \subfigure[Scenario 2]{
        \includegraphics[width =3in]{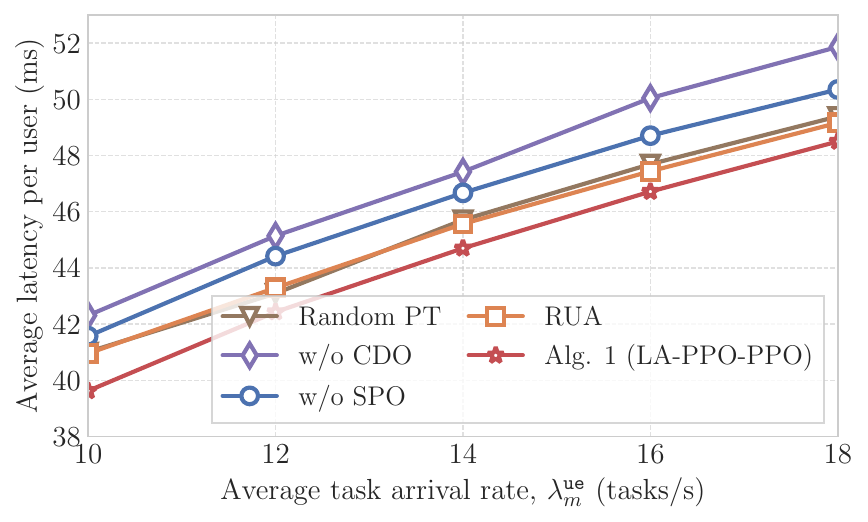}
        \label{fig:comparion_latency_task_arrival_rate_case2}
    }
      \vspace{-2mm}
    \caption{\small Average e2e latency per UE versus the task arrival rate.}
    \label{fig:comparison_latency_task_arrival_rate}
\end{figure}

We evaluate the impact of the average task arrival rate $\lambda^{\mathrm{ue}}_m$ on the average e2e latency in Fig.~\ref{fig:comparison_latency_task_arrival_rate}. We consider  Scenarios~1 and~2 with $M = 100$. As $\lambda^{\mathrm{ue}}_m$ increases from $10$ to $18$, the average latency rises due to increased queueing delays.
As shown in Fig.~\ref{fig:comparison_latency_task_arrival_rate_case1}, Scenario~1, where users are more uniformly distributed, achieves lower latency than Scenario~2 in Fig.~\ref{fig:comparion_latency_task_arrival_rate_case2}, where load imbalance among ESs leads to higher delays.
Notably, at $\lambda^{\mathrm{ue}}_m = 18$, the w/o CDO and w/o SPO schemes exhibit the highest and second-highest latencies, at approximately $52$ ms and $50.5$ ms, respectively. In contrast, Algorithm~\ref{alg_1} consistently outperforms all benchmark schemes, achieving an average latency reduction of about $7.4\%$ over $\lambda^{\mathrm{ue}}_m \in [10, 18]$ compared to the w/o CDO scheme. This highlights the effectiveness and adaptability of the proposed framework under dynamic and imbalanced workloads.

\begin{figure}[t]
	\centering
	\includegraphics[width=0.95\columnwidth,trim={0cm 0.0cm 0cm 0.0cm}]{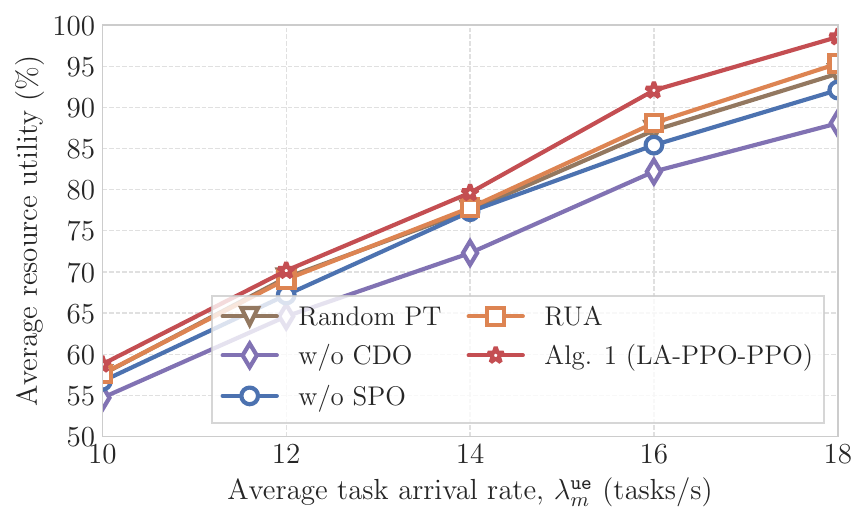}
	\caption{\small Average resource utilization versus the task arrival rate.}
	\label{fig:comparison_resource_utility}
\end{figure}

Fig.~\ref{fig:comparison_resource_utility} shows the average resource utilization versus the task arrival rate for different schemes with $M = 100$. Overall, resource utilization increases with the task arrival rate.
The w/o CDO scheme, which lacks computational delegation optimization, achieves the lowest utilization of approximately $86\%$ at $\lambda_m^\mathtt{ue} = 18$. In contrast, Algorithm~\ref{alg_1} achieves the highest utilization, with an average improvement of about $13\%$ over $\lambda_m^\mathtt{ue} \in [10,18]$ compared to w/o CDO.

\begin{figure}[t]
    \centering
    \vspace{-2mm}
    \subfigure[Comparison with baseline and benchmark schemes]{
        \includegraphics[width = 3in]{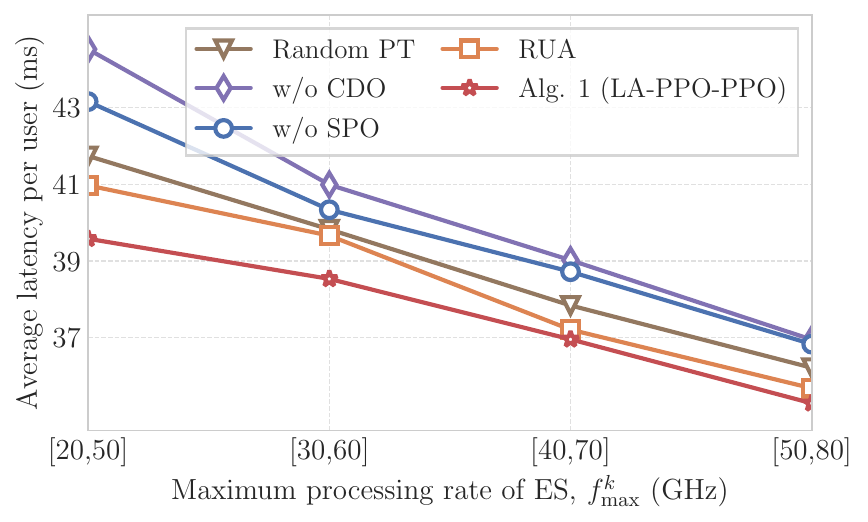}
        \label{fig:comparison_latency_verus_fes}
    }
    \subfigure[Impact on latency and offloading ratio under different UE processing rates]{
        \includegraphics[width =3in]{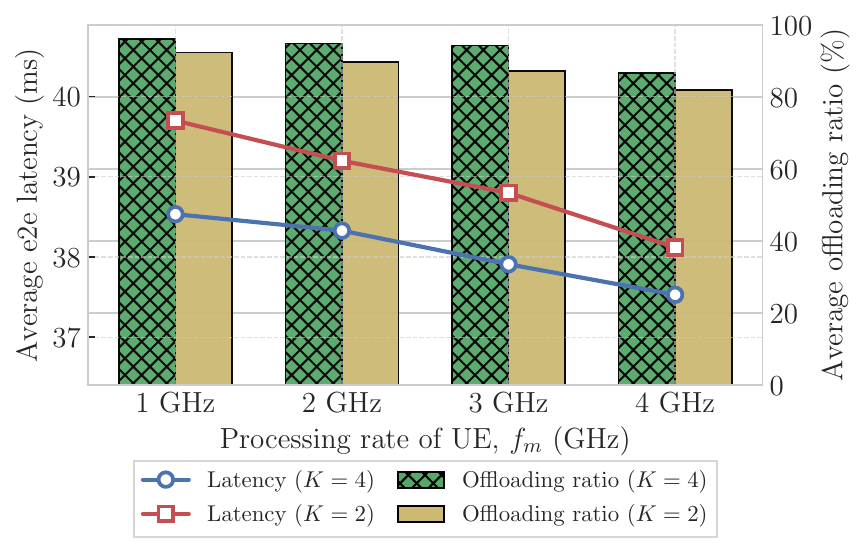}
        \label{fig:comparison_latency_verus_fue}
    }
      \vspace{-2mm}
    \caption{\small Average e2e latency per UE and offloading ratio versus maximum processing rate of ES and UE, with $M = 100$ and $\lambda^{\mathtt{ue}}_m = 10$. }
    \label{fig:comparison_latency_verus_fue_fes}
\end{figure}

Finally, Fig.~\ref{fig:comparison_latency_verus_fes} evaluates the impact of the maximum ES processing rate $f_{\max}^k$ on the average e2e latency under heterogeneous resources, ranging from $[20,30,40,50]$~GHz to $[50,60,70,80]$~GHz for $K=4$. As $f_{\max}^k$ increases, the average latency decreases due to reduced processing delay. Overall, Algorithm~\ref{alg_1} achieves the lowest latency among all compared schemes.
Fig.~\ref{fig:comparison_latency_verus_fue} shows the effect of UE computing capacity on the average e2e latency and offloading ratio over $50$ time frames. Both latency and offloading ratio decrease as the UE processing rate increases, indicating improved local processing capability. Moreover, the offloading ratio with $K=4$ is higher than that with $K=2$, as increased network computing capacity encourages more task offloading to the edge-cloud layer.

\section{Conclusion}\label{sec_conclusion}
In this paper, we investigated joint computational delegation and resource allocation in HECC-aided IoT networks, focusing on reducing the average e2e latency under heterogeneous edge computing capacities. The problem involves tightly coupled decisions across multiple layers, leading to significant challenges in scalability and tractability.
To address these challenges, we developed a two-timescale multi-layer DRL framework with latent action space (2T-MDRL-LA). By decoupling long-term system configuration and short-term resource allocation, the proposed framework effectively handles the interaction between discrete and continuous decisions. In addition, the latent action representation enables efficient exploration of the large combinatorial action space, improving learning stability and scalability.
Numerical results demonstrate that the proposed framework achieves superior performance in terms of latency reduction, resource utilization, and convergence speed compared to benchmark schemes, while closely approaching the optimal performance. These findings highlight the effectiveness of combining multi-timescale optimization with latent-action learning for large-scale and dynamic HECC systems.

\begingroup
\bibliographystyle{IEEEtran}
\bibliography{Bibliography}
\endgroup

\end{document}

%% file: Definition.tex
\usepackage{cite}
\usepackage{subfigure}
\usepackage{amssymb,amsmath}
\usepackage{algorithmic}

\usepackage{color}
\usepackage{epstopdf}
\usepackage{multirow}
\usepackage{multicol}
\usepackage{cases}
\usepackage{setspace}
\usepackage{amsthm}
\usepackage{caption}
\usepackage{mathrsfs}
\usepackage{delarray}
\usepackage{tikz}
\usepackage[ruled]{algorithm2e}
\SetKwInput{KwInput}{Input}
\SetKwInput{KwOutput}{Output}
\SetKwInput{Initialization}{Initialization}
\usepackage{balance}
\usepackage[mathscr]{euscript}
 \let\mathscr\relax
\usepackage[scr]{rsfso}

\allowdisplaybreaks
\usepackage{bbm}

\newtheorem{lemma}{\bf {Lemma}}

\newtheorem{proposition}{{\bf Proposition}}

\newcommand{\st}{{\mathrm{s.t.}}}

\renewcommand{\algorithmicrequire}{\textbf{Input:}}

     \usepackage[compact]{titlesec}
    \titlespacing{\section}{0pt}{2ex}{1ex}
    \titlespacing{\subsection}{0pt}{1ex}{0ex}
    \titlespacing{\subsubsection}{0pt}{0.5ex}{0ex}

\usepackage{capt-of}
\usepackage{dblfloatfix}
\usepackage{varwidth}
\usepackage{tikz}

\usepackage{graphicx}
\usepackage{epstopdf}
\DeclareGraphicsExtensions{.eps,.pdf} 
\graphicspath{ {figures/} }

\usepackage{setspace}
\renewcommand{\baselinestretch}{1} 

\usepackage{amsfonts}
\usepackage{geometry}
\usepackage{array}
\usepackage{booktabs}
\usepackage{siunitx}

\usepackage{babel}
\usepackage{hyperref}

\newcommand{\cp}{\mathsf{cp}}
\newcommand{\cs}{\mathsf{cs}}
\newcommand{\abs}[1]{\left|#1\right|} 
\usepackage{hyperref}
\hypersetup{
 unicode=false,          
 pdftoolbar=true,        
 pdfmenubar=true,        
 pdffitwindow=false,     
 pdfstartview={FitH},    
 pdftitle={My title},    
 pdfauthor={Author},     
 pdfsubject={Subject},   
 pdfcreator={Creator},   
 pdfproducer={Producer}, 
 pdfkeywords={keyword1, key2, key3}, 
 pdfnewwindow=true,      
 colorlinks=true,       
 linkcolor=purple,          
 citecolor=purple,        
 filecolor=purple,      
 urlcolor=purple           
}

\makeatother